\documentclass{alt2027} 

\usepackage{times}
\usepackage{booktabs}
\usepackage{amsfonts}
\usepackage{nicefrac}
\usepackage{microtype}
\usepackage{algorithm}
\usepackage{algorithmic}
\usepackage{multirow}
\usepackage{tikz}
\usepackage[most]{tcolorbox}
\usepackage{dsfont}
\usepackage{enumitem}

\usepackage{amsmath,amsfonts,bm}

\def\eqref#1{eq.~(\ref{#1})}

\def\1{\bm{1}}

\DeclareMathAlphabet{\mathsfit}{\encodingdefault}{\sfdefault}{m}{sl}
\SetMathAlphabet{\mathsfit}{bold}{\encodingdefault}{\sfdefault}{bx}{n}

\DeclareMathOperator*{\argmax}{arg\,max}

\counterwithin{theorem}{section}
\newtheorem{assumption}[theorem]{Assumption}

\renewenvironment{proof}[1][\proofname]{\par\noindent{\bfseries\upshape #1\ }}{\jmlrQED}

\makeatletter
\renewcommand\ps@jmlrtps{%
  \let\@mkboth\@gobbletwo
  \def\@oddhead{}\let\@evenhead\@oddhead
  \def\@oddfoot{\hfill\small\rmfamily\thepage\hfill}\let\@evenfoot\@oddfoot}
\makeatother

\title[Environment-Aware OOD Detection for Vision--Language Models]{A Theoretical Analysis of Environment-Aware Out-of-Distribution Detection for Vision--Language Models}

\author[Sekeh and Wisell]{%
 \Name{Salimeh Sekeh} \Email{ssekeh@sdsu.edu}\\
 \Name{Mary Wisell} \Email{mwisell4707@sdsu.edu}\\
 \addr San Diego State University, San Diego, CA, USA}

\hypersetup{pdftitle={A Theoretical Analysis of Environment-Aware Out-of-Distribution Detection for Vision-Language Models},
            pdfauthor={Salimeh Sekeh, Mary Wisell}}

\begin{document}

\maketitle

\begin{abstract}%
Continual vision--language models (VLMs) encounter a stream of environments and, at deployment, must also reject inputs that do not belong to the environment at hand. 
In this setting, the class information learned in earlier environments is a resource for out-of-distribution (OOD) detection, yet existing detectors score each environment in isolation, and it remains theoretically unclear when the previous environments help or hurt. 
We develop a contribution-theoretic analysis of environment-aware OOD detection built on pairwise cosine expectations, called contribution scores, that measure how the image and text features of one environment align with those of another. 
We propose a cross-environment maximum-concept-matching (MCM) score with two terms: a primary term that matches the test image of the current environment against the class texts of the current environment, and a secondary term that matches the same image against the class texts of a previous environment, weighted by the intra-modal alignment between the two environments. 
We prove that, under a positive environment transition and a bounded class-matching margin for OOD inputs, this score attains a false-positive rate no larger than that of the standard MCM detector at an adjusted threshold, so the secondary term strengthens detection when the environments are strongly connected. 
We then show that the score gap of an MCM detector between two environments equals the self-contribution of the in-distribution environment minus its cross-modal contribution into the OOD environment, which separates soft OOD environments, provably detectable with explicit AUROC, true-positive, and false-positive guarantees, from hard OOD environments, for which any threshold that retains the in-distribution data also accepts OOD inputs at nearly the same guaranteed rate. 
For cascaded detection, in which the text prototypes of all seen environments are pooled, the false-positive rate on every past environment is non-decreasing in time, and we characterize when detection provably fails. 
A contribution balance factor applied to the gradients of the classification loss and of a smoothed false-positive rate completes the framework.
\end{abstract}


\section{Introduction}
\label{sec:intro}
Large vision-language models (VLMs) such as CLIP \citep{radford2021clip} are increasingly deployed in open, changing environments.
They are adapted to new environments by sequential fine-tuning, which trades plasticity for stability: representations acquired for a new environment overwrite those learned for earlier ones, causing catastrophic forgetting \citep{kirkpatrick2017overcoming,wang2024comprehensive}.
At the same time, a deployed model must recognize what it does not know. 
At every step it receives a mixture of in-distribution (ID) data and semantically novel out-of-distribution (OOD) data, and it must classify the former while rejecting the latter \citep{hendrycks2016baseline,bendale2015towards}.

VLMs make this possible without task-specific training: a test image can be scored against the text descriptions of the current classes, as in maximum concept matching (MCM) \citep{ming2022delving}. 
Existing detectors, however, treat every environment in isolation, and existing continual methods, such as prompt tuning \citep{wang2022l2p,wang2022dualprompt}, adapters \citep{gao2021clipadapter}, and weight interpolation \citep{zheng2023zscl}, are designed to limit forgetting rather than to detect novel inputs. 
When environments arrive continuously, the class information of the previous environments is already available and can sharpen detection in the current one. 
Two questions follow: when should a detector use the previous environments, and why does detection that works for one environment fail once several environments have been learned?

We answer both questions with modality-specific contribution scores: pairwise cosine expectations that measure how the image and text features of one environment align with those of another. 
Four scores capture cross-modal ($Con_{\mathcal{I}_t\rightarrow\mathcal{T}_{t'}}$, $Con_{\mathcal{T}_t\rightarrow\mathcal{I}_{t'}}$) and intra-modal ($Con_{\mathcal{I}_t\rightarrow\mathcal{I}_{t'}}$, $Con_{\mathcal{T}_t\rightarrow\mathcal{T}_{t'}}$) structure between environments $\mathcal{E}_t$ and $\mathcal{E}_{t'}$, and each is computable from stored training representations.

Our analysis tells a two-part story. 
The first part concerns how to use the previous environment. 
We propose a cross-environment MCM score with two terms (Algorithm~\ref{alg:env-aware-ood}). 
The primary term computes the similarity between a test image of the current environment $\mathcal{E}_{t'}$ and the class texts of $\mathcal{E}_{t'}$, as in standard MCM. 
The secondary term computes the similarity between the same test image and the class texts of a previous environment $\mathcal{E}_t$. 
The secondary term is added when the two environments are aligned, that is, when their intra-modal contributions are high, and subtracted otherwise. 
We prove that, under a positive environment transition and a bounded class-matching margin for OOD inputs, the cross-environment score attains a false-positive rate no larger than that of the standard MCM detector at an adjusted threshold, so the secondary term enhances OOD detection when there is a strong connection between environments. 
The second part explains when detection succeeds or fails. 
The score gap of an MCM detector between two environments is exactly the self-contribution of the ID environment minus its cross-contribution into the OOD environment. 
This gap classifies OOD environments as soft, for which a threshold inside the gap certifies a high true-positive rate and a low false-positive rate together with an AUROC lower bound, or hard, for which any threshold that retains the ID data accepts OOD data at nearly the same guaranteed rate; strong intra-modal image alignment between two environments provably induces hardness. 
When the prototypes of all seen environments are pooled, detection can fail on earlier environments.
Underlying both parts is a Contribution Balance Factor (CBF), a mask applied to the gradients of the classification loss and of a smoothed false-positive rate that controls how the contributions of one environment transfer to the next across environments and modalities.

Our contributions are as follows:
\begin{itemize}
    \item \textbf{A contribution-theoretic view of continual OOD detection.} We show that the moments of an MCM detector are cross-modal contributions, so the separation between two environments reduces to a single computable quantity, the score gap (Definition~\ref{def:gencon}, Lemma~\ref{lem:ood-moments}).
    \item \textbf{Detection that reuses the previous environment.} We introduce a cross-environment MCM score whose secondary term is weighted by the alignment between environments, and we prove that it does not increase the false-positive rate of standard MCM (Theorem~\ref{thm3:MCM}). We also characterize the optimal contribution balance factor for the classification loss and for a smoothed false-positive rate (Theorems~\ref{thm2:BF-Gradient} and~\ref{thm3:FPR-CBF-Optimal}), and combine both in a training and detection procedure (Algorithm~\ref{alg:env-aware-ood}).
    \item \textbf{A hard versus soft dichotomy for OOD environments.} Soft OOD environments admit explicit AUROC, true-positive, and false-positive guarantees, while hard ones cannot be rejected by any threshold that retains the in-distribution data (Theorems~\ref{thm:soft-ood} and~\ref{thm:hard-ood}). Intra-modal alignment between environments provably induces hardness (Proposition~\ref{prop:align-hard}, Corollary~\ref{cor:aligned-hard}).
    \item \textbf{A failure mode of cascaded detection.} Pooling the prototypes of all seen environments makes the false-positive rate on every past environment non-decreasing in time, and failure concentrates on environments that were learned well or that are aligned with others (Proposition~\ref{prop:pollution}, Theorem~\ref{thm:env-failure}, Corollary~\ref{cor:which-fail}).
\end{itemize}

\section{Related Work}
\label{sec:related}
\textbf{Continual learning.} Continual learning mitigates catastrophic forgetting through regularization \citep{kirkpatrick2017overcoming,zenke2017si,aljundi2018mas}, gradient constraints \citep{lopez2017gradient,farajtabar2020orthogonal,saha2021gpm}, and replay \citep{rebuffi2017icarl,rolnick2019er}; see \citet{wang2024comprehensive} for a survey. For VLMs, prompt-based methods \citep{wang2022l2p,wang2022dualprompt,zhou2022coop,zhou2022cocoop} and adapters \citep{gao2021clipadapter} adapt CLIP while keeping the backbone frozen, and ZSCL \citep{zheng2023zscl} preserves zero-shot transfer by reference-model interpolation. These methods target classification accuracy and do not address the rejection of unknown inputs.

\textbf{OOD detection.} Post-hoc detectors use the maximum softmax probability \citep{hendrycks2016baseline}, temperature scaling with input perturbation \citep{liang2018enhancing}, Mahalanobis distance \citep{lee2018simple}, or free energy \citep{liu2020energy}, refined by activation rectification \citep{sun2021react} and sparsification \citep{sun2022dice}. Training-time methods use auxiliary outliers \citep{hendrycks2018deep} or contrastive representations \citep{tack2020csi}, and open-world recognition couples rejection with the discovery of new classes \citep{bendale2015towards,joseph2021towards,vaze2022generalized}. With VLMs, MCM \citep{ming2022delving} matches an image against class text prototypes, CLIPN \citep{wang2023clipn} learns explicit rejections, and zero-shot variants \citep{fort2021clip,esmaeilpour2022zero} need no task-specific training; near-OOD inputs remain markedly harder than far-OOD ones \citep{fort2021clip}. SCONE \citep{bai2023scone} uses unlabeled wild data, and Temp-SCONE \citep{naiknaware2024tempscone} and T-QPM \citep{naiknaware2025tqpm} extend it to temporal shift. All of these detectors score an environment using only its own classes.

\textbf{Theory of OOD detection.} Prior theory characterizes when OOD detection is learnable \citep{fang2022learnable} and analyzes detection with pretrained VLMs through an information-theoretic lens \citep{peng2026information}, while the modality gap of contrastive VLMs limits cross-modal cosine similarities \citep{liang2022mindgap}. None of these analyses tracks how detection changes across a sequence of environments or how the classes of a previous environment can be used. We provide a contribution-theoretic account of both, with bounds that rest on Cantelli's one-sided inequality \citep{cantelli1928confini}.

\section{Preliminaries}
\label{sec:prelim}

\paragraph{Dynamic Vision-Language Data.} Data arrives from a sequence of environments $\mathcal{E}:=\{\mathcal{E}_1,\ldots,\mathcal{E}_T\}$.
At each time step $t$, the model encounters  $\mathcal{D}_t=\{\mathbf{x}_{t,i},y_{t,i}\}_{i=1}^{n_t}$ from environment $\mathcal{E}_t$. In environment $\mathcal{E}_t$, the data consist of both vision and language modality samples $\mathcal{D}_t=\{\mathcal{I}_t,\mathcal{T}_t\}$, where $\mathcal{I}_t=\{I_{t,i},y_{t,i}\}_{i=1}^{n_t}$ and $\mathcal{T}_t=\{T_{t,i}\}_{i=1}^{n_t}$.

\noindent{\it Text Pattern Construction:} If the ID images are not accompanied with captions, we construct text patterns using the text encoder of the frozen VLM. At time step $t$, for each class $y^{id}_t\in\{1,\ldots,C_t\}$, we employ prompt ensembling over $P$ templates to obtain robust text representations. Let $\{p_{c_t}^{(1)}, \dots, p_{c_t}^{(P)}\}$ denote the set of prompts for class $c_t$ in environment $\mathcal{E}_t$. The class text embedding is computed as:
\begin{align}\label{eq:class-text-embedding}
\mathcal{T}_{c_t} = \mathrm{Normalize}\!\left(\sum_{i=1}^{P}
\mathrm{Normalize}\!\left(\phi^T(p_{c_t}^{(i)})\right)\right),
\end{align}
where $\phi^T$ denotes the frozen CLIP text encoder and
normalization is the $\ell_2$-norm. Collecting all ID class embeddings yields
the \textit{ID text bank}
$\mathcal{T}_t^{id} = [{T}_{t,1}, \dots, T_{t,C_t}]^\top \in \mathbb{R}^{C_t \times d}$, where $d$ is the embedding dimension. Each image $I^{id}_{t,i}$ is paired with a text from the ID text bank. $\mathcal{T}_t^{id}$ is recomputed at each time step as new class labels arrive. \\
If the images are already paired with captions, then $\mathcal{T}_t^{id} = [T_{t,1}, \dots, T_{t,n_t}]\in \mathbb{R}^{n_t \times d}$. \\
\noindent{\it Feature Representations:} CLIP contains an image encoder $f (\cdot)$ and a text encoder $g(\cdot)$, designed to extract features from images and text descriptions respectively. Given the pair (image, text) $(I_{t,i}, T_{t,i})$, the image can
be encoded in $z_{t,i} = f(I_{t,i}) \in \mathbb{R}^{d\times 1}$ and the text can be encoded in $u_{t,i} = g(T_{t,i}) \in \mathbb{R}^{d\times 1}$.

\noindent{\it OOD Detection:} At inference, at each time step $t$, the model receives a new environment $\mathcal{E}_t$ with unlabeled data $\mathcal{D}_t^{id}$ and $\mathcal{D}_t^{ood}$. The model predicts the label of $\mathcal{D}_t^{id}$ and rejects the out-of-distribution (OOD) examples $\mathcal{D}_t^{ood}$. Note that 
$\mathcal{D}_t^{id}=\{\mathcal{I}^{id}_t,\mathcal{T}^{id}_t\}$ and $\mathcal{D}_t^{ood}=\{\mathcal{I}^{ood}_t,\mathcal{T}^{ood}_t\}$ are collections of image–text pairs. At each time step $t$, the model encounters multi-modal data drawn from a non-stationary distribution
\begin{equation}
\mathbb{P}_t= \pi_{t}^{id}\;\mathbb{P}_t^{{id}}
+ \pi_t^{ood}\mathbb{P}_t^{{ood}},\;\;\;\hbox{where}\;\;\;\pi_{t}^{id}+\pi_t^{ood}=1.
\label{eq:wild_method}
\end{equation}
where $\mathbb{P}_t^{id}$ denotes the in-distribution (ID) distribution and $\mathbb{P}_t^{ood}$ denotes the semantic OOD distribution that should be rejected. The detector must simultaneously: 
\textit{i)} maintain classification performance on ID samples, and 
\textit{ii)} reduce the false-positive rate (FPR) and increase the area under the receiver operating characteristic (AUROC) on semantic OOD samples.

\paragraph{Classification Loss.} When no OOD data is present, the regularized loss in (\ref{Loss-function(ID-OOD)}) reduces to the classification loss, which combines a cross-entropy term
\begin{align}\label{Loss-function(ID)}
\mathcal{L}_{CE}(\theta_t):=-\frac{1}{n_t}\sum_{i=1}^{n_t}\log \frac{e^{z^{id}_{t,i}\cdot u^{id}_{t,y_i}/\tau}}{\sum_{c=1}^{C_t} e^{z^{id}_{t,i}\cdot u^{id}_{t,c}/\tau}},
\end{align}
where $u^{id}_{t,c}$ is the text embedding for class label $c$ in ID dataset and $\tau$ is the temperature parameter. 
The CLIP loss is a sum of two cross entropy losses. The first cross entropy term is the image-to-text loss, which is minimized when the text embedding corresponding to an image has a high dot product with that image. The second term inverts this relationship to form a text-to-image loss.
\begin{align}\label{Loss-function(CLIP)}
\mathcal{L}_{CLIP} (\theta_t)=\frac{1}{2n_t}\sum_{i=1}^{n_t}\left[-\log \frac{e^{z^{id}_{t,i}\cdot u^{id}_{t,i}/\tau }}{\sum_{j=1}^{n_t} e^{z^{id}_{t,i}\cdot u^{id}_{t,j}/\tau }}-\log\frac{e^{u^{id}_{t,i}\cdot z^{id}_{t,i}/\tau }}{\sum_{j=1}^{n_t} e^{u^{id}_{t,i}\cdot z^{id}_{t,j}/\tau }}\right].
\end{align}
Combining two losses (\ref{Loss-function(ID)}) and (\ref{Loss-function(CLIP)}), we have loss function (\ref{total-loss}) with weight $\lambda_t$. 
\begin{align}\label{total-loss}
\mathcal{L}(\theta_t)= \mathcal{L}_{CE}(\theta_t)+\lambda_t\; \mathcal{L}_{CLIP} (\theta_t). 
\end{align}

\paragraph{Loss Function with OOD Detection Regularization.} 

At each time step $t$, the parameters $\theta_t$ are learned from ID data $(\mathbf{x}^{id},y)\sim\mathbb{P}^{id}_t$ and OOD data $\mathbf{x}^{ood}\sim\mathbb{P}^{ood}_t$ with semantic shift:
\begin{align}\label{Loss-function(ID-OOD)}
\mathcal{L}^{ood}(\theta_t)=\mathbb{E}_{(\mathbf{x}_t^{id},y_t)\sim \mathbb{P}^{id}_t}\big[\mathcal{L}(\theta_t;\mathbf{x}_t^{id},y_t)\big] +\lambda_{ood}\,\mathbb{E}_{\mathbf{x}_t^{ood}\sim \mathbb{P}^{ood}_t}\big[\mathcal{L}_{reg}(f_{\theta_t}(\mathbf{x}_t^{ood}))\big],
\end{align}
where $\mathcal{L}$ is the classification loss in (\ref{total-loss}), which combines the cross-entropy loss $\mathcal{L}_{CE}$ and the CLIP loss $\mathcal{L}_{CLIP}$, and $\mathcal{L}_{reg}$ is the OOD detection regularization term depending on the detector used, which generally encourages high uncertainty on $\mathbb{P}^{ood}_t$. 

For both losses (\ref{total-loss}) and (\ref{Loss-function(ID-OOD)}), the parameters are obtained as
$$
\theta^*_t={\rm \arg\min}_{\theta_t} \mathcal{L}(\theta_t)\;\;\hbox{or}\;\;{\rm \arg\min}_{\theta_t}\mathcal{L}^{ood}(\theta_t).
$$
In Section~\ref{sec:theory}, the balance-factor results for the classification objective use the loss in (\ref{total-loss}), while the results on OOD detection use the regularized loss in (\ref{Loss-function(ID-OOD)}) together with the detector introduced below.

\paragraph{Cross-Environment OOD Model and Detector.} Writing $P^I_s$ for the marginal law of the images of $\mathcal{E}_s$, we adopt the cross-environment OOD model in which the OOD images of $\mathcal{E}_t$ are images of the other environments,
\begin{equation}
\mathbb{P}^{ood}_t=\sum_{s\neq t}w_{t,s}\,P^I_s,\qquad w_{t,s}\ge 0,\quad \sum_{s\neq t}w_{t,s}=1,
\label{eq:crossood}
\end{equation}
so that the semantic shifts produced by the continual partition serve as OOD data. To detect OOD data, we use a ranking function $g_\theta: \mathcal{X}\mapsto \mathbb{R}$ with parameter $\theta$ and threshold $\lambda$:
\begin{align}\label{OOD-detector}
D_\theta(\mathbf{x})= \left\{ \begin{array}{ll}
\hbox{ID} & \hbox{if}\;\; g_\theta(\mathbf{x}) \geq \lambda,\\
\hbox{OOD}  & \hbox{if}\;\; g_\theta(\mathbf{x}) < \lambda.
\end{array}\right.
\end{align}
The threshold $\lambda$ is typically chosen so that a high fraction of ID data is correctly classified. 
Write $u_{t,c}:=u^{id}_{t,c}$, $c\in\{1,\ldots,C_t\}$, for the class text prototypes of $\mathcal{E}_t$.
The cosine form of the maximum concept matching (MCM) score \citep{ming2022delving} of an image $x$ against $\mathcal{E}_t$ is
\begin{equation}
S_t(x)=\frac{1}{\tau}\max_{c\in\{1,\ldots,C_t\}}\cos\big(u_{t,c},f(x)\big),
\label{eq:mcm}
\end{equation}
which is high for in-distribution and low for OOD images.

\begin{definition}[Generalized and maximal cross-modal contribution]
\label{def:gencon}
For any ordered pair of environments $(t,s)$, not necessarily with $t<s$, let
\begin{align*}
\mathrm{Con}_{\mathcal{T}_t\to\mathcal{I}_s}&:=\mathbb{E}_{x\sim P^I_s}\,\mathbb{E}_{c\sim\mathrm{Unif}\{1,\ldots,C_t\}}\big[\cos(u_{t,c},f(x))\big],\\
\mathrm{Con}^{\max}_{\mathcal{T}_t\to\mathcal{I}_s}&:=\mathbb{E}_{x\sim P^I_s}\Big[\max_{c\in\{1,\ldots,C_t\}} \cos(u_{t,c},f(x))\Big],
\end{align*}
and define the in-distribution alignment of $\mathcal{E}_t$ as $m_t:=\mathrm{Con}^{\max}_{\mathcal{T}_t\to\mathcal{I}_t}$.
\end{definition}
When $s=t'>t$ and the text embedding in Definition~\ref{def:Cross-Modality Contribution} is read as a class prototype of $\mathcal{E}_t$ drawn uniformly, $\mathrm{Con}_{\mathcal{T}_t\to\mathcal{I}_s}$ coincides with $Con_{\mathcal{T}_t\rightarrow\mathcal{I}_{t'}}$; Definition~\ref{def:gencon} removes the ordering restriction and adds the max-aggregated version used by the detector.

\begin{table}[!htbp]
\centering
\caption{Notation. Symbols that appear only inside a single assumption, algorithm, or proof are defined there. The symbol $\lambda$ denotes both the CLIP loss weight in (\ref{total-loss}) ($\lambda_t=\lambda$) and the detection threshold in (\ref{OOD-detector}). The meaning is fixed by context.}
\label{tab:notation}
\begin{tabular}{lp{0.42\linewidth}l}
\hline
\textbf{Symbol} & \textbf{Meaning} & \textbf{Defined in}\\
\hline
\textit{Environments and data} & & \\
$\mathcal{E}_t$, $t=1,\ldots,T$ & environment at step $t$ (with $t<t'$) & Sec.~\ref{sec:prelim}\\
$\mathcal{D}_t$, $n_t$, $C_t$ & data, sample size, and classes of $\mathcal{E}_t$ & Sec.~\ref{sec:prelim}\\
$\mathcal{I}_t$, $\mathcal{T}_t$ & images and texts of $\mathcal{E}_t$ & Sec.~\ref{sec:prelim}\\
$\mathbb{P}^{id}_t$, $\mathbb{P}^{ood}_t$ & ID and OOD data distributions & (\ref{eq:wild_method})\\
$P^I_s$, $w_{t,s}$ & image law of $\mathcal{E}_s$; OOD mixture weight & (\ref{eq:crossood})\\
\hline
\textit{Encoders and losses} & & \\
$f$, $g$ & image and text encoders & Sec.~\ref{sec:prelim}\\
$z_{t,i}$, $u_{t,i}$ & image and text embeddings & Sec.~\ref{sec:prelim}\\
$u_{t,c}$ & text prototype of class $c$ in $\mathcal{E}_t$ & (\ref{Loss-function(ID)})\\
$\tau$ & temperature & (\ref{Loss-function(ID)})\\
$\mathcal{L}_{CE}$, $\mathcal{L}_{CLIP}$ & cross-entropy and CLIP losses & (\ref{Loss-function(ID)}), (\ref{Loss-function(CLIP)})\\
$\mathcal{L}(\theta_t)$, $\lambda_t$ & classification loss and its CLIP weight & (\ref{total-loss})\\
$\mathcal{L}^{ood}$, $\mathcal{L}_{reg}$, $\lambda_{ood}$ & regularized loss; regularizer; weight & (\ref{Loss-function(ID-OOD)})\\
\hline
\textit{Contributions and balance factor} & & \\
$Con_{\mathcal{I}_t\rightarrow\mathcal{I}_{t'}}$, $Con_{\mathcal{T}_t\rightarrow\mathcal{T}_{t'}}$ & intra-modal contributions & Def.~\ref{def:Intra-Modality Contribution}\\
$Con_{\mathcal{I}_t\rightarrow\mathcal{T}_{t'}}$, $Con_{\mathcal{T}_t\rightarrow\mathcal{I}_{t'}}$ & cross-modal contributions & Def.~\ref{def:Cross-Modality Contribution}\\
$Con_{\mathcal{E}_t\rightarrow\mathcal{E}_{t'}}$, $\alpha$ & environment contribution; threshold & Def.~\ref{def:Cross-environment alignment}\\
$\mathrm{Con}^{\max}_{\mathcal{T}_t\to\mathcal{I}_s}$, $m_t$ & maximal contribution; ID alignment & Def.~\ref{def:gencon}\\
$\mathbf{P}_{t\rightarrow t'}$, $\mathbf{P}^{(i)}_{t\rightarrow t'}$ & balance factor; per-sample version & Def.~\ref{def:CBF}\\
$S^{id}_{t,y_i}$ & intra-environment alignment score & (\ref{eq:scores-intra-environment})\\
$FPR_{t\rightarrow t'}$, $FPR_{t'}$, $n^{ood}_{t'}$ & FPRs of the two MCM scores; OOD count & (\ref{eq:fpr})\\
$\widehat{FPR}_{t\rightarrow t'}$, $\beta$ & smoothed FPR; sharpness & (\ref{eq:soft-fpr})\\
\hline
\textit{OOD detection} & & \\
$D_\theta$, $g_\theta$, $\lambda$ & detector; ranking function; threshold & (\ref{OOD-detector})\\
$s_i(\mathcal{I}^{test}_{t'})$, $s_i(\mathcal{I}^{test}_{t\rightarrow t'})$ & cosine with class text $i$ of $\mathcal{E}_{t'}$ and of $\mathcal{E}_t$ & (\ref{def.cross-environment-MCM-score})\\
$S_{MCM}(\mathcal{I}^{test}_{t})$ & intra-environment MCM score & (\ref{def.MCM-score})\\
$S_{MCM}(\mathcal{I}^{test}_{t\rightarrow t'})$ & cross-environment MCM score & (\ref{def.cross-environment-MCM-score})\\
$\mathbf{A}^*_{t\rightarrow t'}$ & environment alignment factor & (\ref{Env.alignment-factor})\\
$S_t(x)$, $\ell$ & cosine MCM score; threshold $\ell=\tau\lambda$ & (\ref{eq:mcm})\\
$\mathrm{TPR}_t$, $\mathrm{FPR}_t$, $\mathrm{AUROC}_t$ & rates for ID $\mathcal{E}_t$ vs.\ OOD $\mathcal{E}_s$ & (\ref{eq:tpr-fpr})\\
$\gamma_{t,s}$, $\nu_{t,s}$, $\eta$ & score gap; spread; soft/hard margin & (\ref{eq:gap-contribution})\\
$\delta_t$ & true-class misalignment & (\ref{eq:true-class-misalign})\\
$U_k$, $S^{(k)}$, $\rho_{k,s}$ & pooled prototypes, score; pollution & (\ref{eq:pooled}), (\ref{eq:rho})\\
\hline
\end{tabular}
\end{table}
\section{Theory}\label{sec:theory}
We analyze in continual VLMs over a sequence of environments $\mathcal{E}:=\{\mathcal{E}_1,\mathcal{E}_2,\ldots, \mathcal{E}_T\}$ how the intra-modality and cross-modality contributions govern the transfer between environments and the detection of OOD examples. We state the intra- and cross-modal contribution measurements in our analysis. 

\begin{definition} \label{def:Intra-Modality Contribution}
(Intra-Modal Contribution) The image $\mathcal{I}^{id}_t$ from environment $\mathcal{E}_t$ with class $C_t$, $\sigma_{\mathcal{I}}$-contributes in feature representation of images $\mathcal{I}^{id}_{t'}$ in environment $\mathcal{E}_{t'}$ if and only if
\begin{align}
Con_{\mathcal{I}_t\rightarrow\mathcal{I}_{t'}}:=\mathbb{E}_{(\mathbf{x}^{id}_{t'},y_{t'})\sim \mathbb{P}^{id}_{t'}}\left[\cos(\tilde{z}^{id}_{t'},z^{id}_{t})\right]=(1-\sigma_{\mathcal{I}}),\;\; (t<t').
\end{align}
Similarly the text $\mathcal{T}^{id}_t$ from environment $\mathcal{E}_t$ with class $C_t$, $\sigma_{\mathcal{T}}$-contributes in feature representation of text $\mathcal{T}^{id}_{t'}$ in environment $\mathcal{E}_{t'}$ if and only if
\begin{align}
Con_{\mathcal{T}_t\rightarrow\mathcal{T}_{t'}}:=\mathbb{E}_{(\mathbf{x}^{id}_{t'},y_{t'})\sim \mathbb{P}^{id}_{t'}}\left[\cos(\tilde{u}^{id}_{t'},u^{id}_{t})\right]=(1-\sigma_{\mathcal{T}}),\;\;(t<t').
\end{align} 
\end{definition}
\begin{definition} \label{def:Cross-Modality Contribution}
(Cross-Environment Cross-Modal Contribution) The image $\mathcal{I}^{id}_t$ from environment $\mathcal{E}_t$ with class $C_t$, $\sigma_{\mathcal{IT}}$-contributes in feature representation of text $\mathcal{T}^{id}_{t'}$ in environment $\mathcal{E}_{t'}$ if and only if
\begin{align}
Con_{\mathcal{I}_t\rightarrow\mathcal{T}_{t'}}:=\mathbb{E}_{(\mathbf{x}^{id}_{t'},y_{t'})\sim \mathbb{P}^{id}_{t'}}\left[\cos(\tilde{u}^{id}_{t'},z^{id}_{t})\right]=(1-\sigma_{\mathcal{IT}}), \;\; (t<t'),
\end{align}
and the text $\mathcal{T}^{id}_t$ from environment $\mathcal{E}_t$ with class $C_t$, $\sigma_{\mathcal{TI}}$-contributes in feature representation of images $\mathcal{I}^{id}_{t'}$ in environment $\mathcal{E}_{t'}$
\begin{align}
Con_{\mathcal{T}_t\rightarrow\mathcal{I}_{t'}}:=\mathbb{E}_{(\mathbf{x}^{id}_{t'},y_{t'})\sim \mathbb{P}^{id}_{t'}}\left[\cos(u^{id}_{t},\tilde{z}^{id}_{t'})\right]=(1-\sigma_{\mathcal{TI}}),\;\; (t<t').
\end{align}
\end{definition}
In both Def.~\ref{def:Intra-Modality Contribution} and \ref{def:Cross-Modality Contribution}, $\cos(\cdot,\cdot)$ is cosine similarity and $0\leq\sigma_{\mathcal{I}}, \sigma_{\mathcal{T}},\sigma_{\mathcal{IT}},\sigma_{\mathcal{TI}}\leq1$.

\begin{definition} \label{def:Cross-environment alignment}
Two environments $\mathcal{E}_t$ and $\mathcal{E}_{t'}$ are aligned if and only if $Con_{\mathcal{E}_t\rightarrow\mathcal{E}_{t'}}:=Con_{\mathcal{I}_t\rightarrow\mathcal{I}_{t'}}+Con_{\mathcal{T}_t\rightarrow\mathcal{T}_{t'}}>\alpha$ for a threshold $\alpha$, with the intra-modal contributions of Def.~\ref{def:Intra-Modality Contribution}.
\end{definition}
\subsection{Cross-modality contribution balance factor bridges the gap between environments.}

\noindent We next show how controlling the balance of cross-modality contributions affects the transition between environments. We mainly seek to investigate the central question: \textit{How modal contribution balance factor controls intra- and cross-modal contribution between environments?}

\begin{definition}\label{def:CBF}
(Cross-Modal Contribution Balance Factor) We define contribution balance factor (CBF) from $\mathcal{E}_{t}$ to $\mathcal{E}_{t'}$  by $\mathbf{P}_{t\rightarrow t'}\in\{0,1\}^d$ which is optimized by
\begin{align*}
\mathbf{P}_{t\rightarrow t'}^*={\rm argmax}_{\mathbf{P}_{t\rightarrow t'}\in\{0,1\}^d} \frac{\|\mathbf{P}_{t\rightarrow t'}\odot\nabla_\theta \mathcal{L}(\theta_{t'})\|^2}{\|\nabla_\theta \mathcal{L}(\theta_{t'})\|^2}, \;\;\hbox{where $\odot$ is pairwise multiplication.}
\end{align*}
We define the cross-modal CBF from $\mathcal{E}_{t}$ to sample $(\mathbf{x}_{t',i},y_{t',i})$ as $\mathbf{P}^{(i)}_{t\rightarrow t'}$. 
\end{definition}
The assumptions below play distinct roles: Assumptions~\ref{ass:Sufficient class-normalized gradient drift} and~\ref{ass:Sufficient cross-modal gradient drift} focus on the normalized components of $\mathcal{L}_{CE}$ and $\mathcal{L}_{CLIP}$, assuming that their gradient drift is sufficiently large, whereas Assumption~\ref{ass:slow-gradient-magnitude} bounds the overall gradient magnitude of the loss function.
\begin{assumption}[Sufficient class-normalized gradient drift]\label{ass:Sufficient class-normalized gradient drift}
We assume that the cross-modal CBF is sufficiently effective in average over all labels:
$$
\mathbf{P}_{t\rightarrow t'}\odot\nabla_\theta \log\alpha^C_{t',i}=\Omega(\xi_{t'}), \;\;\;\hbox{where $\alpha^C_{t,i}:=\sum_{c=1}^C e^{z^{id}_{t,i}\cdot u^{id}_{t,c}/\tau}$. }
$$
\end{assumption}
\begin{assumption}[Sufficient cross-modal gradient drift]\label{ass:Sufficient cross-modal gradient drift}
 We assume that the cross-modal CBF is sufficiently effective in average over all samples.
\begin{align}
\frac{1}{n_{t'}}\sum_{i=1}^{n_{t'}} \mathbf{P}_{t\rightarrow t'}\odot\nabla_\theta \log\alpha^{z}_{t',i}=\Omega(\xi_{z,t'}), \;\;\;\hbox{where $\alpha_{t,i}^{z}:= \sum_{j=1}^{n_t} e^{z^{id}_{t,i}\cdot u^{id}_{t,j}/\tau }$. }\\
\frac{1}{n_{t'}}\sum_{i=1}^{n_{t'}} \mathbf{P}_{t\rightarrow t'}\odot\nabla_\theta \log\alpha^{u}_{t',i}=\Omega(\xi_{u,t'}), \;\;\;\hbox{where $\alpha_{t,i}^{u}:= \sum_{j=1}^{n_t} e^{u^{id}_{t,i}\cdot z^{id}_{t,j}/\tau }$. }
\end{align}
\end{assumption}
\begin{assumption}[Slow gradient]
\label{ass:slow-gradient-magnitude}
The magnitude of the gradient of loss changes slowly i.e.
$\|\nabla_\theta \mathcal{L}(\theta_{t})\|^2=\mathcal{O}(1/\xi)$, which implies $1/\|\nabla_\theta \mathcal{L}(\theta_{t})\|^2=\Omega(\xi)$.
\end{assumption}
\begin{theorem}[Optimal CBF via intra-environment alignment scores]\label{thm2:BF-Gradient}
Under Assumptions~\ref{ass:Sufficient class-normalized gradient drift}-\ref{ass:slow-gradient-magnitude}, the optimal $\mathbf{P}_{t\rightarrow t'}^*$ for the classification loss $\mathcal{L}(\theta_{t'})$ in (\ref{total-loss}) is alternatively achieved by 
\begin{align}\label{thm:CBF-optimization}
\mathbf{P}_{t\rightarrow t'}^*=\frac{\Omega(\xi)}{2n_{t'}}\sum_{i=1}^{n_{t'}} \argmax_{\mathbf{P}^{(i)}_{t\rightarrow t'}} S^{id}_t [\mathbf{P}^{(i)}_{t\rightarrow t'}],  
\end{align}
where $\xi$ is constant and  
$$
S^{id}_t [\mathbf{P}^{(i)}_{t\rightarrow t'}]:=\|\mathbf{P}^{(i)}_{t\rightarrow t'}\odot\nabla_\theta S^{id}_{t',y_i}\|^2+\lambda\|\mathbf{P}^{(i)}_{t\rightarrow t'}\odot\nabla_\theta S^{id}_{t',z_i\rightarrow u_i}\|^2+\lambda\|\mathbf{P}^{(i)}_{t\rightarrow t'}\odot\nabla_\theta S^{id}_{t',u_i\rightarrow z_i}\|^2,
$$
and intra-environment alignment scores $S^{id}_{t',y_i}$, $S^{id}_{t',z_i\rightarrow u_i}$, and $S^{id}_{t',u_i\rightarrow z_i}$ are defined as follows
\begin{align}\label{eq:scores-intra-environment}
S^{id}_{t,y_i}=(z^{id}_{t,i}\cdot u^{id}_{t,y_i})/\tau, \;\; S^{id}_{t,z_i\rightarrow u_i}=(z^{id}_{t,i}\cdot u^{id}_{t,i})/\tau,\;\;\; \hbox{and}\;\;\; S^{id}_{t,u_i\rightarrow z_i}=(u^{id}_{t,i}\cdot z^{id}_{t,i})/\tau.
\end{align}
\end{theorem}
\noindent{\bf Remark.} Note that the scores in (\ref{eq:scores-intra-environment}) differ from the cross-modal contributions defined in Def.~\ref{def:Cross-Modality Contribution}, even though both are based on dot products. Intra-environment alignment captures relationships within a single environment and does not aggregate information across environments. In contrast, our cross-modal contribution scores bridge vision and language embeddings across evolving environments, facilitating the transfer of cross-modal information to subsequent learning stages.

In the regularized loss $\mathcal{L}^{ood}$ in (\ref{Loss-function(ID-OOD)}), we take the regularization term $\mathbb{E}_{\mathbf{x}_t^{ood}\sim\mathbb{P}^{ood}_t}[\mathcal{L}_{reg}(f_{\theta_t}(\mathbf{x}_t^{ood}))]$ to be the false-positive rate of the cross-environment MCM detector, defined as
\begin{align}\label{eq:fpr}
FPR_{t\rightarrow t'}(\tau,\lambda)=\frac{1}{n^{ood}_{t'}} \sum_{i=1}^{n^{ood}_{t'}} \mathbb{1}[S_{MCM}(\mathcal{I}^{ood}_{i, t\rightarrow t'})\geq \lambda].
\end{align}
$FPR$ is the fraction of $n^{ood}_{t'}$ OOD images $\mathcal{I}^{ood}$ wrongly accepted as ID. To take the derivative from FPR, we use soft FPR surrogate: 
\begin{align}
\label{eq:soft-fpr}
\widehat{FPR}_{t\rightarrow t'} (\tau,\lambda)= \frac{1}{n^{ood}_{t'}} \sum_{i=1}^{n^{ood}_{t'}}\sigma\left(\frac{S_{MCM}(\mathcal{I}^{ood}_{i, t\rightarrow t'})-\lambda}{\beta}\right), 
\end{align}
where $\beta$ controls the sharpness and $\sigma$ is sigmoid function. 

\begin{assumption}[FPR Slow gradient]
\label{ass:FPR slow-gradient-magnitude}
The magnitude of the gradient of FPR changes slowly i.e.
$\|\nabla_\theta \widehat{FPR}_{t\rightarrow t'}\|^2=\mathcal{O}(1/\xi_{t\rightarrow t'})$, which implies $1/\|\nabla_\theta \widehat{FPR}_{t\rightarrow t'}\|^2=\Omega(\xi_{t\rightarrow t'})$.
\end{assumption}
\begin{theorem}[Optimal CBF for the smoothed FPR]\label{thm3:FPR-CBF-Optimal} Under Assumption~\ref{ass:FPR slow-gradient-magnitude}, the optimal $\mathbf{P}_{t\rightarrow t'}^*$ for OOD detection regularization i.e. $FPR$ (OOD loss) is given by
\begin{align}\label{thm:FPR-CBF-optimization}
\mathbf{P}_{t\rightarrow t'}^*=\frac{C_{\sigma,k} \Omega(\xi_{t\rightarrow t'})}{\beta n^{ood}_{t'}}\sum_{i=1}^{n^{ood}_{t'}} argmax_{\mathbf{P}^{(i)}_{t\rightarrow t'}} \Delta_{t\rightarrow t'}[\mathbf{P}^{(i)}_{t\rightarrow t'}],
\end{align}
where 
\begin{align}
\Delta_{t\rightarrow t'}[\mathbf{P}^{(i)}_{t\rightarrow t'}]&= \large\|\mathbf{P}^{(i)}_{t\rightarrow t'}\odot [\nabla_\theta s_{k, t'}-\sum_{j=1}^C p_{j, t'} \nabla_\theta s_{j, t'} ]\large\|^2\nonumber\\
&+\mathbf{A}^*_{t\rightarrow t'} \large\|\mathbf{P}^{(i)}_{t\rightarrow t'}\odot[\nabla_\theta s_{k, t\rightarrow t'}-\sum_{j=1}^C p_{j, t\rightarrow t'} \nabla_\theta s_{j, t\rightarrow t'}] \large\|^2-\hat{\rho}_{max, t\rightarrow t'}, 
\end{align}
where, for the $i$-th OOD image, $s_{j,t'}$ and $s_{j,t\rightarrow t'}$ are the similarities $s_j(\mathcal{I}^{ood}_{i,t'})$ and $s_j(\mathcal{I}^{ood}_{i,t\rightarrow t'})$ in (\ref{def.cross-environment-MCM-score}), $p_{j,t'}$ and $p_{j,t\rightarrow t'}$ are the corresponding softmax probabilities, $k$ is the class attaining the maximum in (\ref{def.cross-environment-MCM-score}), and $C_{\sigma,k}$, $\hat{\rho}_{max, t\rightarrow t'}$, and $\xi_{t\rightarrow t'}$ are constants.

\end{theorem}

\noindent Algorithm~\ref{alg:env-aware-ood} combines the two balance factors of Theorems~\ref{thm2:BF-Gradient} and~\ref{thm3:FPR-CBF-Optimal} into one training procedure and uses the cross-environment MCM score of Definition~\ref{def:Cross-MCM} at test time. Both objectives $S^{id}_t[\mathbf{P}^{(i)}_{t\rightarrow t'}]$ and $\Delta_{t\rightarrow t'}[\mathbf{P}^{(i)}_{t\rightarrow t'}]$ are linear in the binary mask, since $\|\mathbf{P}\odot v\|^2=\sum_{j}P_j v_j^2$; under a budget of $b$ retained coordinates, each per-sample maximization therefore keeps the $b$ coordinates with the largest weights. The secondary text bank of the previous environment is used only through the alignment factor $\mathbf{A}^*_{t\rightarrow t'}$, which adds its concept matches when $\mathcal{E}_t$ and $\mathcal{E}_{t'}$ are aligned and subtracts them otherwise.

\begin{algorithm}[t]
\caption{Contribution-balanced continual training and environment-aware OOD detection}
\label{alg:env-aware-ood}
\begin{algorithmic}[1]
\REQUIRE environments $\mathcal{E}_1,\ldots,\mathcal{E}_T$; encoders $f$, $g$ with parameters $\theta$; temperature $\tau$; weights $\lambda_t$, $\lambda_{ood}$; sharpness $\beta$; alignment threshold $\alpha$; mask budget $b$; step size $\eta_{lr}$; detection threshold $\lambda$
\FOR{$t'=1,\ldots,T$}
\IF{$t'>1$}
\STATE $t\leftarrow t'-1$; compute $Con_{\mathcal{I}_t\rightarrow\mathcal{I}_{t'}}$, $Con_{\mathcal{T}_t\rightarrow\mathcal{T}_{t'}}$ from stored representations \hfill (Def.~\ref{def:Intra-Modality Contribution})
\STATE $Con_{\mathcal{E}_t\rightarrow\mathcal{E}_{t'}}\leftarrow Con_{\mathcal{I}_t\rightarrow\mathcal{I}_{t'}}+Con_{\mathcal{T}_t\rightarrow\mathcal{T}_{t'}}$
\STATE $\mathbf{A}^*_{t\rightarrow t'}\leftarrow +Con_{\mathcal{E}_t\rightarrow\mathcal{E}_{t'}}$ if $Con_{\mathcal{E}_t\rightarrow\mathcal{E}_{t'}}>\alpha$, and $-Con_{\mathcal{E}_t\rightarrow\mathcal{E}_{t'}}$ otherwise \hfill (\ref{Env.alignment-factor})
\ENDIF
\REPEAT
\STATE $\mathbf{P}^{(i)}\leftarrow\arg\max_{\|\mathbf{P}\|_0=b} S^{id}_{t}[\mathbf{P}]$ for each ID sample $i$ \hfill (Theorem~\ref{thm2:BF-Gradient})
\STATE $\mathbf{P}^{id}\leftarrow\frac{1}{n_{t'}}\sum_{i=1}^{n_{t'}}\mathbf{P}^{(i)}$
\STATE $G\leftarrow\mathbf{P}^{id}\odot\nabla_\theta\mathcal{L}(\theta_{t'})$ \hfill (\ref{total-loss})
\IF{$t'>1$}
\STATE $\mathbf{P}^{(i)}\leftarrow\arg\max_{\|\mathbf{P}\|_0=b} \Delta_{t\rightarrow t'}[\mathbf{P}]$ for each OOD sample $i$ \hfill (Theorem~\ref{thm3:FPR-CBF-Optimal})
\STATE $\mathbf{P}^{ood}\leftarrow\frac{1}{n^{ood}_{t'}}\sum_{i=1}^{n^{ood}_{t'}}\mathbf{P}^{(i)}$
\STATE $G\leftarrow G+\lambda_{ood}\,\mathbf{P}^{ood}\odot\nabla_\theta\widehat{FPR}_{t\rightarrow t'}(\tau,\lambda)$ \hfill (\ref{eq:soft-fpr})
\ENDIF
\STATE $\theta\leftarrow\theta-\eta_{lr}\,G$
\UNTIL{convergence on $\mathcal{D}_{t'}$}
\STATE store the training image and text representations of $\mathcal{E}_{t'}$
\ENDFOR
\STATE \textbf{Detection at step $t'$:} for a test image $\mathcal{I}^{test}_{t'}$, set $z_{t'}\leftarrow f(\mathcal{I}^{test}_{t'})$
\STATE $s_i(\mathcal{I}^{test}_{t'})\leftarrow\cos(z_{t'},u_{t',i})$ for $i=1,\ldots,C$ \hfill (primary term)
\STATE $s_i(\mathcal{I}^{test}_{t\rightarrow t'})\leftarrow\cos(z_{t'},u_{t,i})$ for $i=1,\ldots,C$ \hfill (secondary term)
\STATE \textbf{return} ID if $S_{MCM}(\mathcal{I}^{test}_{t\rightarrow t'})\ge\lambda$ and OOD otherwise \hfill (\ref{def.cross-environment-MCM-score})
\end{algorithmic}
\end{algorithm}

\subsection{Environment-Aware Zero-Shot OOD Detection}
We assume that the CLIP (or VLM) has been trained in each environment using (\ref{total-loss}) or (\ref{Loss-function(ID-OOD)}) and the training image and text representations are stored for contribution computations.

{\bf Cross-Modal Contribution and MCM Score}

\begin{definition}\label{def:Intra-MCM}
(Intra-Environment MCM Score~\cite{ming2022delving}) 

At timestep $t$, for a given task with label set $\mathcal{Y}_{in} = \{y_1, y_2, ..., y_C \}$, and a collection of concept vectors $\mathcal{T} (t_i)$, $i\in \{1, 2, ..., C\}$, where $t_i$ is the text prompt ``\texttt{this is a photo of a --label--}" for a label $y_i$. At time step $t$, for any test image $\mathcal{I}_t^{test}$, 
the MCM score is defined by
\begin{align}\label{def.MCM-score}
S_{MCM}(\mathcal{I}^{test}_t)= \max_i\frac{e^{s_i(\mathcal{I}_t^{test})/\tau}}{\sum_{j=1}^C e^{s_j(\mathcal{I}_t^{test})/\tau}},
\end{align}
where $\tau$ is the temperature and $s_i$ is the cosine similarity between the image embedding $z_t=f(\mathcal{I}_t^{test})$ and the text embedding $u_{t,i}:=g(\mathcal{T}_t^{train}(t_i))$. Recall the OOD detector $D_\theta$ in (\ref{OOD-detector}) and replace $g_\theta$ with $S_{MCM}$ to compute MCM-based OOD detector. 
\end{definition}
We define cross-environment MCM score under the setup that ID inputs are accompanied by text descriptions (label captions for ID data) as follows:
\begin{definition}

\label{def:Cross-MCM}
(Cross-Environment MCM Score) Consider two environments $\mathcal{E}_t$ and $\mathcal{E}_{t'}$, $t<t'$ with label sets $\mathcal{Y}_{in,t} = \{y_{t,1}, y_{t,2}, ..., y_{t,C} \}$ and $\mathcal{Y}_{in,t'} = \{y_{t',1}, y_{t',2}, ..., y_{t',C} \}$ and collection of concepts vectors $\mathcal{T}_{t,i}$, $\mathcal{T}_{t',i}$, for $i=1,\ldots,C$. Without loss of generality we assume that both environments have the same class size. At time step $t'$, for any test image input $\mathcal{I}^{test}_{t'}$, the cross-environment MCM score is defined as 
\begin{align}\label{def.cross-environment-MCM-score}
S_{MCM}(\mathcal{I}^{test}_{t\rightarrow t'})= \max_i\left\{\frac{e^{s_i(\mathcal{I}_{t'}^{test})/\tau}}{\sum_{j=1}^C e^{s_j(\mathcal{I}_{t'}^{test})/\tau}}+\mathbf{A}^*_{t\rightarrow t'}\frac{e^{s_i(\mathcal{I}_{t\rightarrow t'}^{test})/\tau}}{\sum_{j=1}^C e^{s_j(\mathcal{I}_{t\rightarrow t'}^{test})/\tau}}\right\},
\end{align}
where $s_i(\mathcal{I}_{t'}^{test})$ is the cosine similarity between test image encoder $z_{t'}=f(\mathcal{I}_{t'}^{test})$ and train text encoder in $\mathcal{E}_{t'}$ $u_{t',i}:=g(\mathcal{T}_{t'}^{train}(t_i))$.  $s_i(\mathcal{I}_{t\rightarrow t'}^{test})$ is the cosine similarity between test image encoder  $z_{t'}=f(\mathcal{I}_{t'}^{test})$ in $\mathcal{E}_{t'}$ and train text encoder in $\mathcal{E}_{t}$ $u_{t,i}:=g(\mathcal{T}_{t}^{train}(t_i))$. The environment alignment factor $\mathbf{A}^*_{t\rightarrow t'}$ is determined by intra-modal contribution in Def.~\ref{def:Intra-Modality Contribution}: 
\begin{align}\label{Env.alignment-factor}
\mathbf{A}^*_{t\rightarrow t'}= \left\{ \begin{array}{rcl}
+Con_{\mathcal{E}_t\rightarrow \mathcal{E}_{t'}} & \hbox{if $\mathcal{E}_t$ and $\mathcal{E}_{t'}$ are aligned}\;\; \\
\\ -Con_{\mathcal{E}_t\rightarrow \mathcal{E}_{t'}}  & \hbox{if $\mathcal{E}_t$ and $\mathcal{E}_{t'}$ are not aligned},  
\end{array}\right.
\end{align}
where $Con_{\mathcal{E}_t\rightarrow \mathcal{E}_{t'}}$ and alignment are as in Definition~\ref{def:Cross-environment alignment}. 
\end{definition}
{\bf Remark:} The similarity score $s_i(\mathcal{I}_{t\rightarrow t'}^{test})$ in (\ref{def.cross-environment-MCM-score}), is the cross-modal contribution in Def.~\ref{def:Cross-Modality Contribution} in the presence of a test image when image encoder $f$ and text encoder $g$ are trained on training image in environment $t'$ and text in environment $t$, respectively.  
\begin{assumption}[Positive Environment Transition]\label{ass:positive environment transition}
We assume that the two environments $\mathcal{E}_t$ and $\mathcal{E}_{t'}$, $t<t'$ have a positive transition if the softmax normalizer of the secondary term is at least that of the primary term, i.e.
$$
\sum_{j=1}^C e^{s_j(\mathcal{I}_{t'}^{test})/\tau}\leq \sum_{j=1}^C e^{s_j(\mathcal{I}_{t\rightarrow t'}^{test})/\tau}. 
$$
\end{assumption}

\begin{assumption}[Optimal Class Connections] \label{ass:Optimal Class Connections}
Define $\hat{y}_{t\rightarrow t'}:= \argmax_{i\in [C]}s_i(\mathcal{I}^{test}_{t\rightarrow t'})$ to be largest cross-environment cosine similarities for an OOD input $\mathcal{I}$. And $\hat{y}_{t'}:= \argmax_{i\in [C]}s_i(\mathcal{I}^{test}_{t'})$ to be the largest intra-environment cosine similarities for an OOD input $\mathcal{I}$. Then $\exists \delta\geq 0$ and $\exists M\geq 0$. 
\begin{align}
&(a)\;\;\mathbb{P}^{ood}_{\mathcal{I}^{test}_{t'}} \left(\frac{1}{C}\sum_{j=1}^{C}\big[s_{\hat{y}_{t'}}(\mathcal{I}_{t'}^{test})-s_j(\mathcal{I}_{t'}^{test})\big]< \delta\right)=1,\\
&(b) \;\;\mathbb{P}^{ood}_{\mathcal{I}^{test}_{t'}} \left(s_{\hat{y}_{t\rightarrow t'}}(\mathcal{I}^{test}_{t\rightarrow t'})< \frac{1+M}{Con_{\mathcal{E}_t\rightarrow\mathcal{E}_{t'}}} s_{\hat{y}_{t'}}(\mathcal{I}^{test}_{t'})-\frac{\kappa}{Con_{\mathcal{E}_t\rightarrow\mathcal{E}_{t'}}}\right)=1,
\end{align}
where $Con_{\mathcal{E}_t\rightarrow\mathcal{E}_{t'}}=2-\sigma_{\mathcal{I}}-\sigma_{\mathcal{T}}$ (Def.~\ref{def:Cross-environment alignment}) and $\kappa$ is an upper bound for $s_{\hat{y}_{t\rightarrow t'}}(\mathcal{I}^{test}_{t'})$.
\end{assumption}
\begin{theorem}[Cross-environment contribution improves FPR.]\label{thm3:MCM}
Consider a pre-trained CLIP-like model. If the OOD distribution $\mathbb{P}^{ood}_{\mathcal{I}^{test}_{t'}}$ of test images in $\mathcal{E}_{t'}$ satisfies Assumptions~\ref{ass:positive environment transition} and \ref{ass:Optimal Class Connections}, then 
\begin{align}
FPR_{t\rightarrow t'}(\tau,\lambda)\leq FPR_{t'}(\tau,\lambda_{t'}), \;\; \hbox{where $\lambda_{t'}$ is a function of $M$, $\tau$, and $\delta$.}
\end{align}
\end{theorem}

\subsection{Hard vs Soft OOD Examples}
Not all OOD inputs are equally difficult to reject: images from an environment whose content is close to the ID environment are hard, whereas semantically distant ones are soft. We make this precise with Definition~\ref{def:gencon}. Throughout, the detector $D_\theta$ in (\ref{OOD-detector}) uses $g_\theta=S_t$ in (\ref{eq:mcm}) with threshold $\lambda$, and $\ell:=\tau\lambda\in(-1,1)$ is the threshold in cosine units. For an ID environment $\mathcal{E}_t$ and an OOD environment $\mathcal{E}_s$, $s\neq t$, let $X\sim P^I_t$ and $X'\sim P^I_s$ be independent and define
\begin{align}
\mathrm{TPR}_t(\lambda):=\Pr\big[S_t(X)\ge\lambda\big],\quad \mathrm{FPR}_t(s;\lambda):=\Pr\big[S_t(X')\ge\lambda\big],\label{eq:tpr-fpr}\\
\mathrm{AUROC}_t(s):=\Pr\big[S_t(X)>S_t(X')\big]+\tfrac{1}{2}\Pr\big[S_t(X)=S_t(X')\big].\label{eq:auroc}
\end{align}
Let $\gamma_{t,s}:=\tau\big(\mathbb{E}[S_t(X)]-\mathbb{E}[S_t(X')]\big)$ be the score gap, $v^{id}_t:=\mathrm{Var}(\tau S_t(X))$ and $v^{ood}_{t,s}:=\mathrm{Var}(\tau S_t(X'))$, and $\nu_{t,s}:=v^{id}_t+v^{ood}_{t,s}\le 2$ the score spread. Proofs of the results in this subsection are given in Appendix~\ref{app:hard-soft}.

\begin{lemma}[Detector moments are contributions]\label{lem:ood-moments}
For any $s\neq t$, $\tau\,\mathbb{E}[S_t(X)]=m_t$ and $\tau\,\mathbb{E}[S_t(X')]=\mathrm{Con}^{\max}_{\mathcal{T}_t\to\mathcal{I}_s}\ge\mathrm{Con}_{\mathcal{T}_t\to\mathcal{I}_s}$; hence
\begin{equation}\label{eq:gap-contribution}
\gamma_{t,s}=m_t-\mathrm{Con}^{\max}_{\mathcal{T}_t\to\mathcal{I}_s}=\mathrm{Con}^{\max}_{\mathcal{T}_t\to\mathcal{I}_t}-\mathrm{Con}^{\max}_{\mathcal{T}_t\to\mathcal{I}_s}.
\end{equation}
\end{lemma}

\begin{definition}[Hard and soft OOD]\label{def:hard-soft}
Let $\eta>0$. An environment $\mathcal{E}_s$ is $\eta$-soft OOD for $\mathcal{E}_t$ if $\gamma_{t,s}\ge\eta$, and $\eta$-hard OOD if $\gamma_{t,s}\le\eta$; it is strictly hard if $\gamma_{t,s}<0$. By Lemma~\ref{lem:ood-moments}, $\mathrm{Con}_{\mathcal{T}_t\to\mathcal{I}_s}\ge m_t-\eta$ suffices for $\eta$-hardness.
\end{definition}

\begin{theorem}[Soft OOD is provably separable]\label{thm:soft-ood}
If $\mathcal{E}_s$ is $\eta$-soft OOD for $\mathcal{E}_t$, then $\mathrm{AUROC}_t(s)\ge\eta^{2}/(\eta^{2}+\nu_{t,s})\ge\eta^{2}/(\eta^{2}+2)$, and at the threshold $\ell=m_t-\eta/2$,
\begin{equation}\label{eq:soft-rates}
\mathrm{TPR}_t(\lambda)\ge\frac{\eta^{2}/4}{v^{id}_t+\eta^{2}/4}\ge\frac{\eta^{2}}{\eta^{2}+4},\qquad \mathrm{FPR}_t(s;\lambda)\le\frac{v^{ood}_{t,s}}{v^{ood}_{t,s}+\eta^{2}/4}\le\frac{4}{4+\eta^{2}}.
\end{equation}
\end{theorem}

\begin{theorem}[Hard OOD cannot be reliably rejected]\label{thm:hard-ood}
If $\mathcal{E}_s$ is $\eta$-hard OOD for $\mathcal{E}_t$, then the certificate of Theorem~\ref{thm:soft-ood} is at most $\eta^{2}/(\eta^{2}+\nu_{t,s})$, and for every threshold with $\ell<m_t-\eta$,
\begin{equation}\label{eq:hard-rates}
\mathrm{FPR}_t(s;\lambda)\ge\frac{m_t-\eta-\ell}{1-\ell},\qquad \mathrm{TPR}_t(\lambda)\ge\frac{m_t-\ell}{1-\ell},
\end{equation}
so the guaranteed acceptance rates of ID and hard OOD images differ by at most $\eta/(1-\ell)$. If $\mathcal{E}_s$ is strictly hard, then $\mathrm{AUROC}_t(s)\le\nu_{t,s}/\big(\nu_{t,s}+(\mathrm{Con}^{\max}_{\mathcal{T}_t\to\mathcal{I}_s}-m_t)^{2}\big)$.
\end{theorem}

Theorems~\ref{thm:soft-ood} and~\ref{thm:hard-ood} split OOD environments by the single number (\ref{eq:gap-contribution}): for soft OOD, a threshold inside the gap certifies a high TPR and a low FPR simultaneously, whereas for hard OOD any threshold that retains the ID data also accepts OOD data at nearly the same guaranteed rate. Hardness is induced by intra-modal alignment. Let
\begin{equation}\label{eq:true-class-misalign}
\delta_t:=1-\mathbb{E}_{(x,y)\sim P^{id}_t}\big[\cos(u_{t,y},f(x))\big]
\end{equation}
be the true-class misalignment of $\mathcal{E}_t$; for $\ell_2$-normalized embeddings, $\delta_t=1-\tau\,\mathbb{E}[S^{id}_{t,y_i}]$ with $S^{id}_{t,y_i}$ in (\ref{eq:scores-intra-environment}).

\begin{proposition}[Intra-modal alignment induces hardness]\label{prop:align-hard}
If $(x,y)\sim P^{id}_t$ and $x'\sim P^I_s$ are coupled so that $\mathbb{E}[\cos(f(x),f(x'))]=Con_{\mathcal{I}_t\rightarrow\mathcal{I}_{s}}=1-\sigma_{\mathcal{I}}$ (Definition~\ref{def:Intra-Modality Contribution}), then
\begin{equation}\label{eq:align-hard}
\mathrm{Con}^{\max}_{\mathcal{T}_t\to\mathcal{I}_s}\ge 1-\big(\sqrt{\delta_t}+\sqrt{\sigma_{\mathcal{I}}}\big)^{2}.
\end{equation}
\end{proposition}

\begin{corollary}[Aligned environments are hard OOD]\label{cor:aligned-hard}
If $Con_{\mathcal{I}_t\rightarrow\mathcal{I}_{s}}\ge 1-\bar\sigma$ and $1-\big(\sqrt{\delta_t}+\sqrt{\bar\sigma}\big)^{2}\ge m_t-\eta$, then $\mathcal{E}_s$ is $\eta$-hard OOD for $\mathcal{E}_t$ and (\ref{eq:hard-rates}) applies.
\end{corollary}

Corollary~\ref{cor:aligned-hard} explains the sign of $\mathbf{A}^*_{t\rightarrow t'}$ in (\ref{Env.alignment-factor}): images of an aligned environment behave as near-ID under $S_t$, so their concept matches are added rather than subtracted. Because pre-trained CLIP separates image and text embeddings \citep{liang2022mindgap}, $\delta_t$ is large before adaptation, and (\ref{eq:align-hard}) becomes informative as adaptation reduces the true-class misalignment.

\subsection{Cascaded OOD Detection Failure}
In continual deployment, detection cascades over the environments seen so far. Environments arrive in an order $\pi=(\pi_1,\dots,\pi_T)$, and after step $k$ the detector uses the pooled prototype set and score
\begin{equation}\label{eq:pooled}
U_k=\bigcup_{j\le k}\{u_{\pi_j,c}:c=1,\ldots,C_{\pi_j}\},\qquad S^{(k)}(x)=\frac{1}{\tau}\max_{u\in U_k}\cos\big(u,f(x)\big).
\end{equation}
At step $k$, images of $\mathcal{E}_{\pi_k}$ are ID and images of a past environment $\mathcal{E}_{\pi_s}$, $s<k$, are OOD. With $X\sim P^I_{\pi_k}$ and $X'\sim P^I_{\pi_s}$ independent, $\mathrm{FPR}^{(k)}(s;\lambda)$, $\mathrm{AUROC}^{(k)}(s)$, $\gamma^{(k)}_s$, and $\nu^{(k)}_s\le 2$ are defined as in (\ref{eq:tpr-fpr})--(\ref{eq:auroc}) with $S^{(k)}$ in place of $S_t$. Let
\begin{equation}\label{eq:rho}
\rho_{k,s}:=\max\Big\{m_{\pi_s},\,\max_{j\le k,\,j\neq s}\mathrm{Con}^{\max}_{\mathcal{T}_{\pi_j}\to\mathcal{I}_{\pi_s}}\Big\},\qquad M_k:=\tau\,\mathbb{E}\big[S^{(k)}(X)\big]\in\big[m_{\pi_k},1\big],
\end{equation}
be the pollution level of $\mathcal{E}_{\pi_s}$ and the mean pooled ID score. Proofs are given in Appendix~\ref{app:cascade}.

\begin{proposition}[Monotone prototype pollution]\label{prop:pollution}
For every image $x$, step $k$, and threshold $\lambda$: $S^{(k)}(x)\le S^{(k+1)}(x)$; $\mathrm{FPR}^{(k)}(s;\lambda)\le\mathrm{FPR}^{(k+1)}(s;\lambda)$ for every past environment $s<k$; and $\tau\,\mathbb{E}_{x\sim P^I_{\pi_s}}[S^{(k)}(x)]\ge\rho_{k,s}$ for every $s\le k$.
\end{proposition}

{\bf Environment-Dependent Failure.} Once an environment's prototypes enter the pool, its images score at least its own alignment, so failure is a property of each past environment.

\begin{theorem}[Environment-dependent failure of cascaded detection]\label{thm:env-failure}
For every past environment $s<k$, $\gamma^{(k)}_s\le M_k-\rho_{k,s}$. If $\rho_{k,s}>M_k$, then $\mathrm{AUROC}^{(k)}(s)\le\nu^{(k)}_s/\big(\nu^{(k)}_s+(\rho_{k,s}-M_k)^{2}\big)$. For every threshold with $\ell=\tau\lambda<\rho_{k,s}$ and every later step $k'\ge k$,
\begin{equation}\label{eq:cascade-fpr}
\mathrm{FPR}^{(k')}(s;\lambda)\ge\frac{\rho_{k,s}-\ell}{1-\ell}\ge\frac{m_{\pi_s}-\ell}{1-\ell}.
\end{equation}
\end{theorem}

\begin{corollary}[Which environments fail]\label{cor:which-fail}
Let $s<k$ and $\ell<m_{\pi_s}$. Every past environment is falsely accepted at rate at least $(m_{\pi_s}-\ell)/(1-\ell)$ at every later step, so better-learned environments fail more. Moreover, by Proposition~\ref{prop:align-hard}, $\rho_{k,s}\ge\max_{j\le k,\,j\neq s}\big[1-\big(\sqrt{\delta_{\pi_j}}+\sqrt{\sigma_{\mathcal{I}}(\pi_j,\pi_s)}\big)^{2}\big]$ with $1-\sigma_{\mathcal{I}}(\pi_j,\pi_s)=Con_{\mathcal{I}_{\pi_j}\rightarrow\mathcal{I}_{\pi_s}}$, so past environments that are image-aligned with any seen environment satisfy (\ref{eq:cascade-fpr}) with a large $\rho_{k,s}$.
\end{corollary}

Cascaded failure is therefore a property of the environments and their order rather than of the detector: it concentrates on well-learned or aligned past environments and, at a fixed threshold, can only worsen over time. Uniform pooling is the source of the failure, which motivates scores such as (\ref{def.cross-environment-MCM-score}) that weight other environments by their alignment instead of pooling them.

\section{Conclusion}\label{sec:conclusion}
We studied out-of-distribution detection for vision--language models that encounter a sequence of environments. 
Our analysis rests on pairwise contribution scores. 
These scores are cosine expectations between the image and text features of two environments. 
We showed that they govern when detection succeeds or fails.

First, we proposed a cross-environment MCM score. 
The score adds the concept matches of a previous environment when the two environments are aligned. 
It subtracts them when the environments are not aligned. Under the assumptions of Theorem~\ref{thm3:MCM}, this score does not increase the false-positive rate of the standard MCM detector. 
A contribution balance factor applied to the gradients of the classification loss and of a smoothed false-positive rate gives the training procedure in Algorithm~\ref{alg:env-aware-ood}.

Second, we showed that the score gap of an MCM detector equals the self-contribution of the ID environment minus its cross-modal contribution into the OOD environment. 
This gap separates soft OOD environments from hard ones. 
For a soft OOD environment, a threshold inside the gap certifies a high true-positive rate, a low false-positive rate, and an AUROC lower bound. 
A hard OOD environment cannot be reliably rejected by any threshold that retains the in-distribution data. 
Intra-modal alignment between two environments provably makes one hard OOD for the other.

Third, we studied cascaded detection, in which the prototypes of all seen environments are pooled. 
At a fixed threshold, the false-positive rate on every past environment can only grow over time. 
Failure concentrates on environments that were learned well or that are aligned with other seen environments.

Our guarantees rely on assumptions about the transition between environments, the class-matching margin of OOD inputs, and the magnitude of the loss gradient. 
We do not verify these assumptions here. 
The separability bounds are distribution-free and therefore conservative. 
Tighter bounds under sub-Gaussian score tails, an analysis that combines more than one previous environment, and an empirical evaluation of Algorithm~\ref{alg:env-aware-ood} are natural next steps.

\acks{This work has been partially supported by NSF CAREER CCF-2451457. The findings are those of the authors only and do not represent any position of these funding bodies.}

\bibliography{egbib}

@article{peng2026information,
  title={An information-theoretical framework for understanding out-of-distribution detection with pretrained vision-language models},
  author={Peng, Bo and Lu, Jie and Zhang, Guangquan and Fang, Zhen},
  journal={Advances in Neural Information Processing Systems},
  volume={38},
  pages={1930--1957},
  year={2026}
}

@article{ming2022delving,
  title={Delving into out-of-distribution detection with vision-language representations},
  author={Ming, Yifei and Cai, Ziyang and Gu, Jiuxiang and Sun, Yiyou and Li, Wei and Li, Yixuan},
  journal={Advances in neural information processing systems},
  volume={35},
  pages={35087--35102},
  year={2022}
}

@article{wang2024comprehensive,
  title={A comprehensive survey of continual learning: Theory, method and application},
  author={Wang, Liyuan and Zhang, Xingxing and Su, Hang and Zhu, Jun},
  journal={IEEE transactions on pattern analysis and machine intelligence},
  volume={46},
  number={8},
  pages={5362--5383},
  year={2024},
  publisher={IEEE}
}

@InProceedings{hendrycks2016baseline,
  author = "D. Hendrycks and K. Gimpel",
  title = "A Baseline for Detecting Misclassified and Out-of-Distribution Examples in Neural Networks",
  booktitle = "Proc. ICLR",
  year = "2017",
}

@InProceedings{liang2018enhancing,
  author = "S. Liang and Y. Li and R. Srikant",
  title = "Enhancing the Reliability of Out-of-Distribution Image Detection in Neural Networks",
  booktitle = "Proc. ICLR",
  year = "2018",
}

@InProceedings{lee2018simple,
  author = "K. Lee and K. Lee and H. Lee and J. Shin",
  title = "A Simple Unified Framework for Detecting Out-of-Distribution Samples and Adversarial Attacks",
  booktitle = "Proc. NeurIPS",
  year = "2018",
}

@InProceedings{liu2020energy,
  author = "W. Liu and X. Wang and J. Owens and Y. Li",
  title = "Energy-Based Out-of-Distribution Detection",
  booktitle = "Proc. NeurIPS",
  year = "2020",
}

@InProceedings{sun2021react,
  author = "Y. Sun and C. Guo and Y. Li",
  title = "ReAct: Out-of-Distribution Detection with Rectified Activations",
  booktitle = "Proc. NeurIPS",
  year = "2021",
}

@InProceedings{sun2022dice,
  author = "Y. Sun and Y. Li",
  title = "DICE: Leveraging Sparsification for Out-of-Distribution Detection",
  booktitle = "Proc. ECCV",
  year = "2022",
}

@InProceedings{hendrycks2018deep,
  author = "D. Hendrycks and M. Mazeika and T. Dietterich",
  title = "Deep Anomaly Detection with Outlier Exposure",
  booktitle = "Proc. ICLR",
  year = "2019",
}

@InProceedings{tack2020csi,
  author = "J. Tack and S. Mo and J. Jeong and J. Shin",
  title = "CSI: Novelty Detection via Contrastive Learning on Distributionally Shifted Instances",
  booktitle = "Proc. NeurIPS",
  year = "2020",
}

@InProceedings{bai2023scone,
  author = "H. Bai and Z. Ming and Y. Li",
  title = "Feed Two Birds with One Scone: Exploiting Wild Data for Both Out-of-Distribution Generalization and Detection",
  booktitle = "Proc. ICML",
  year = "2023",
}

@InProceedings{radford2021clip,
  author = "A. Radford and J.~W. Kim and C. Hallacy and A. Ramesh and G. Goh and S. Agarwal and G. Sastry and A. Askell and P. Mishkin and J. Clark and G. Krueger and I. Sutskever",
  title = "Learning Transferable Visual Models From Natural Language Supervision",
  booktitle = "Proc. ICML",
  year = "2021",
}

@InProceedings{fort2021clip,
  author = "S. Fort and J. Ren and B. Lakshminarayanan",
  title = "Exploring the Limits of Out-of-Distribution Detection",
  booktitle = "Proc. NeurIPS",
  year = "2021",
}

@InProceedings{wang2023clipn,
  author = "H. Wang and Y. Li and H. Yao and X. Li",
  title = "CLIPN for Zero-Shot Out-of-Distribution Detection: Teaching CLIP to Say No",
  booktitle = "Proc. ICCV",
  year = "2023",
}

@InProceedings{esmaeilpour2022zero,
  author = "S. Esmaeilpour and B. Liu and E. Robertson and L. Shu",
  title = "Zero-Shot Out-of-Distribution Detection Based on the Pre-trained Model CLIP",
  booktitle = "Proc. AAAI",
  year = "2022",
}

@article{kirkpatrick2017overcoming,
  author = "J. Kirkpatrick and R. Pascanu and N. Rabinowitz and J. Veness and G. Desjardins and A.~A. Rusu and K. Milan and J. Quan and T. Ramalho and A. Grabska-Barwinska and D. Hassabis and C. Clopath and D. Kumaran and R. Hadsell",
  title = "Overcoming Catastrophic Forgetting in Neural Networks",
  journal = "PNAS",
  volume = "114",
  number = "13",
  pages = "3521-3526",
  year = "2017",
}

@InProceedings{zenke2017si,
  author = "F. Zenke and B. Poole and S. Ganguli",
  title = "Continual Learning Through Synaptic Intelligence",
  booktitle = "Proc. ICML",
  year = "2017",
}

@InProceedings{aljundi2018mas,
  author = "R. Aljundi and F. Babiloni and M. Elhoseiny and M. Rohrbach and T. Tuytelaars",
  title = "Memory Aware Synapses: Learning What (Not) to Forget",
  booktitle = "Proc. ECCV",
  year = "2018",
}

@InProceedings{lopez2017gradient,
  author = "D. Lopez-Paz and M. Ranzato",
  title = "Gradient Episodic Memory for Continual Learning",
  booktitle = "Proc. NeurIPS",
  year = "2017",
}

@InProceedings{farajtabar2020orthogonal,
  author = "M. Farajtabar and N. Azizan and A. Mott and A. Li",
  title = "Orthogonal Gradient Descent for Continual Learning",
  booktitle = "Proc. AISTATS",
  year = "2020",
}

@InProceedings{saha2021gpm,
  author = "G. Saha and I. Garg and K. Roy",
  title = "Gradient Projection Memory for Continual Learning",
  booktitle = "Proc. ICLR",
  year = "2021",
}

@InProceedings{rebuffi2017icarl,
  author = "S.-A. Rebuffi and A. Kolesnikov and G. Sperl and C.~H. Lampert",
  title = "iCaRL: Incremental Classifier and Representation Learning",
  booktitle = "Proc. CVPR",
  year = "2017",
}

@InProceedings{rolnick2019er,
  author = "D. Rolnick and A. Ahuja and J. Schwarz and T. Lillicrap and G. Wayne",
  title = "Experience Replay for Continual Learning",
  booktitle = "Proc. NeurIPS",
  year = "2019",
}

@InProceedings{wang2022l2p,
  author = "Z. Wang and Z. Zhang and C.-Y. Lee and H. Zhang and R. Sun and X. Ren and G. Su and V. Perot and J. Dy and T. Pfister",
  title = "Learning to Prompt for Continual Learning",
  booktitle = "Proc. CVPR",
  year = "2022",
}

@InProceedings{wang2022dualprompt,
  author = "Z. Wang and Z. Zhang and H. Zhang and C.-Y. Lee and R. Sun and X. Ren and G. Su and V. Perot and J. Dy and T. Pfister",
  title = "DualPrompt: Complementary Prompting for Rehearsal-Free Continual Learning",
  booktitle = "Proc. ECCV",
  year = "2022",
}

@InProceedings{zheng2023zscl,
  author = "Z. Zheng and M. Ma and K. Wang and Z. Qin and X. Yue and Y. You",
  title = "Preventing Zero-Shot Transfer Degradation in Continual Learning of Vision-Language Models",
  booktitle = "Proc. ICCV",
  year = "2023",
}

@InProceedings{bendale2015towards,
  author = "A. Bendale and T. Boult",
  title = "Towards Open World Recognition",
  booktitle = "Proc. CVPR",
  year = "2015",
}

@InProceedings{joseph2021towards,
  author = "K.~J. Joseph and S. Khan and F.~S. Khan and V. Balasubramanian",
  title = "Towards Open World Object Detection",
  booktitle = "Proc. CVPR",
  year = "2021",
}

@InProceedings{vaze2022generalized,
  author = "S. Vaze and K. Han and A. Vedaldi and A. Zisserman",
  title = "Generalized Category Discovery",
  booktitle = "Proc. CVPR",
  year = "2022",
}

@misc{naiknaware2024tempscone,
  author = "A. Naiknaware and S.~Y. Sekeh",
  title = "Temp-SCONE: A Novel Out-of-Distribution Detection and Domain Generalization Framework for Wild Data with Temporal Shift",
  note = "arXiv:2512.04571",
  year = "2024",
}

@misc{naiknaware2025tqpm,
  author = "A. Naiknaware and S.~Y. Sekeh",
  title  = "T-QPM: Enabling Temporal Out-of-Distribution Detection and Domain Generalization for Vision-Language Models in Open-World",
  note = "arXiv:2603.18481",
  year = "2025",
}

@misc{gao2021clipadapter,
      title={CLIP-Adapter: Better Vision-Language Models with Feature Adapters}, 
      author={Peng Gao and Shijie Geng and Renrui Zhang and Teli Ma and Rongyao Fang and Yongfeng Zhang and Hongsheng Li and Yu Qiao},
      year={2025},
      eprint={2110.04544},
      archivePrefix={arXiv},
      primaryClass={cs.CV},
      url={https://arxiv.org/abs/2110.04544}, 
}

@article{zhou2022coop,
  author = "K.~Zhou and J.~Yang and C.~C.~Loy and Z.~Liu",
  title = "Learning to Prompt for Vision-Language Models",
  journal = "IJCV",
  volume = "130",
  number = "9",
  pages = "2337-2348",
  year = "2022",
}

@InProceedings{zhou2022cocoop,
  author = "Kaiyang Zhou and Jingkang Yang and Chen Change Loy and Ziwei Liu",
  title = "Conditional Prompt Learning for Vision-Language Models",
  booktitle = "Proc. CVPR",
  pages = "16816-16825",
  year = "2022",
}

@InProceedings{cantelli1928confini,
  author = "F.~P.~Cantelli",
  title = "Sui confini della probabilit{\`a}",
  booktitle = "Atti del Congresso Internazionale dei Matematici, Bologna, 3--10 Settembre 1928",
  volume = "6",
  pages = "47--59",
  year = "1928",
}

@InProceedings{fang2022learnable,
  author = "Z.~Fang and Y.~Li and J.~Lu and J.~Dong and B.~Han and F.~Liu",
  title = "Is Out-of-Distribution Detection Learnable?",
  booktitle = "Proc. NeurIPS",
  year = "2022",
}

@InProceedings{liang2022mindgap,
  author = "W.~Liang and Y.~Zhang and Y.~Kwon and S.~Yeung and J.~Zou",
  title = "Mind the Gap: Understanding the Modality Gap in Multi-modal Contrastive Representation Learning",
  booktitle = "Proc. NeurIPS",
  year = "2022",
}

\appendix

\section{Proofs}\label{app:proofs}
\subsection{Proof of Theorem~\ref{thm2:BF-Gradient}}
\label{app:disscuss_thm_2}
\begin{proof}
Recall the total loss function $\mathcal{L}(\theta_t)$ for $\mathcal{E}_t$ in (\ref{total-loss}) and take gradient w.r.t parameter $\theta$
\begin{align}\label{eq:thm2-1}
\nabla_\theta\mathcal{L}(\theta_t)= \nabla_\theta\mathcal{L}_{CE}(\theta_t)+\lambda_t\; \nabla_\theta\mathcal{L}_{CLIP} (\theta_t). 
\end{align}
Next step is to apply CBF from $\mathcal{E}_t$ to  $\mathcal{E}_{t'}$, i.e. $\mathbf{P}_{t\rightarrow t'}\odot\nabla_\theta \mathcal{L}(\theta_{t'})$ from Def.~\ref{def:CBF} in both sides of (\ref{eq:thm2-1}). 
$$
\mathbf{P}_{t\rightarrow t'}\odot\nabla_\theta \mathcal{L}(\theta_{t'})=\mathbf{P}_{t\rightarrow t'}\odot \nabla_\theta\mathcal{L}_{CE}+\lambda_t \mathbf{P}_{t\rightarrow t'}\odot \nabla_\theta\mathcal{L}_{CLIP} (\theta_t).
$$
To compute $\mathbf{P}_{t\rightarrow t'}\odot \nabla_\theta\mathcal{L}_{CE}$, we use the definition of $\mathcal{L}_{CE}(\theta_t)$ at time step $t$ and its simplified formulation

\begin{align}\label{eq:thm2-3}
\mathcal{L}_{CE}(\theta_t)&= -\frac{1}{n_t}\sum_{i=1}^{n_t}(z^{id}_{t,i}\cdot u^{id}_{t,y_i}/\tau)+\frac{1}{n_t}\sum_{i=1}^{n_t}\log \alpha^C_{t,i},\;\;\hbox{where}\;\; \alpha^{C}_{t,i}=\sum_{c=1}^C e^{z^{id}_{t,i}\cdot u^{id}_{t,c}/\tau}\nonumber\\
&=-\frac{1}{n_t}\sum_{i=1}^{n_t} S^{id}_{t, y_i}+\frac{1}{n_t}\sum_{i=1}^{n_t}\log \alpha^C_{t,i}, \;\;\hbox{where score}\;\;   S^{id}_{t,y_i}=(z^{id}_{t,i}\cdot u^{id}_{t,y_i})/\tau.
\end{align}
Take the gradient of (\ref{eq:thm2-3}) w.r.t $\theta$ and apply CBF:
\begin{align}\label{eq:thm2-8}
\nabla_\theta\mathcal{L}_{CE}(\theta_t)=-\frac{1}{n_t}\sum_{i=1}^{n_t} \nabla_\theta S^{id}_{t,y_i}+\frac{1}{n_t}\sum_{i=1}^{n_t} \nabla_\theta \log\alpha^C_{t,i}.
\end{align}
\begin{align}\label{eq:thm2-4}
\mathbf{P}_{t\rightarrow t'}\odot \nabla_\theta\mathcal{L}_{CE}(\theta_{t'})&=-\frac{1}{n_{t'}}\sum_{i=1}^{n_{t'}} \mathbf{P}^{(i)}_{t\rightarrow t'}\odot\nabla_\theta S^{id}_{t',y_i}+\frac{1}{n_{t'}}\sum_{i=1}^{n_{t'}}\mathbf{P}_{t\rightarrow t'}\odot\nabla_\theta \log\alpha^C_{t',i}\nonumber\\
&= -\frac{1}{n_{t'}}\sum_{i=1}^{n_{t'}} \mathbf{P}^{(i)}_{t\rightarrow t'}\odot\nabla_\theta S^{id}_{t',y_i}+\Omega(\xi_{t'}).
\end{align}
The last equality in (\ref{eq:thm2-4}) is driven from the Assumption~\ref{ass:Sufficient class-normalized gradient drift}. Now let us discuss the second term in (\ref{eq:thm2-1}): Recall the $\mathcal{L}_{CLIP}$ and simplify it:

\begin{align}\label{eq:thm2-9}
L_{CLIP}(\theta_t)=\frac{1}{2n_t}\sum_{i=1}^{n_t}\left[-z^{id}_{t,i}\cdot u^{id}_{t,i}/\tau-u^{id}_{t,i}\cdot z^{id}_{t,i}/\tau+\log \alpha_{t,i}^{z}+\log \alpha_{t,i}^{u}\right],
\end{align}
$$
\hbox{where}\;\;\alpha_{z,t}^{z}:= \sum_{j=1}^{n_t} e^{z^{id}_{t,i}\cdot u^{id}_{t,j}/\tau }\;\;\;\hbox{and}\;\;\; \alpha_{t,i}^{u}:= \sum_{j=1}^{n_t} e^{u^{id}_{t,i}\cdot z^{id}_{t,j}/\tau }.
$$
Taking the gradient w.r.t. $\theta$, applying CBF, and using the Assumption~\ref{ass:Sufficient cross-modal gradient drift} implies
\begin{align}\label{eq:thm2-5}
\mathbf{P}_{t\rightarrow t'}\odot \nabla_\theta\mathcal{L}_{CLIP}(\theta_{t'})=-\frac{1}{2n_{t'}}\sum_{i=1}^{n_{t'}} \left[\mathbf{P}^{(i)}_{t\rightarrow t'}\odot\nabla_\theta S^{id}_{t',z_i\rightarrow u_i}+\mathbf{P}^{(i)}_{t\rightarrow t'}\odot\nabla_\theta S^{id}_{t',u_i\rightarrow z_i}\right]+\Omega(\bar{\xi}_{t'}).
\end{align}
Here $\Omega(\bar{\xi}_{t'})=\Omega(\xi_{z,t'})+\Omega(\xi_{u,t'})$, and two scores $S^{id}_{t,z_i\rightarrow u_i}$ and $S^{id}_{t,u_i\rightarrow z_i}$ are defined as follows, 
$$
S^{id}_{t,z_i\rightarrow u_i}=(z^{id}_{t,i}\cdot u^{id}_{t,i})/\tau,\;\;\; \hbox{and}\;\;\; S^{id}_{t,u_i\rightarrow z_i}=(u^{id}_{t,i}\cdot z^{id}_{t,i})/\tau.
$$
Putting (\ref{eq:thm2-4}) and (\ref{eq:thm2-5}) together leads to
\begin{align}\label{eq:thm2-6}
&\mathbf{P}_{t\rightarrow t'}\odot\nabla_\theta \mathcal{L}(\theta_{t'})=-\frac{1}{n_{t'}}\sum_{i=1}^{n_{t'}} \mathbf{P}^{(i)}_{t\rightarrow t'}\odot\nabla_\theta S^{id}_{t',y_i}\nonumber\\
&-\frac{\lambda}{2n_{t'}}\sum_{i=1}^{n_{t'}} \left[\mathbf{P}^{(i)}_{t\rightarrow t'}\odot\nabla_\theta S^{id}_{t',z_i\rightarrow u_i}+\mathbf{P}^{(i)}_{t\rightarrow t'}\odot\nabla_\theta S^{id}_{t',u_i\rightarrow z_i}\right]+\Omega(\widetilde{\xi}_{t'}), 
\end{align}
where $\Omega(\widetilde{\xi}_{t'})=\lambda\Omega(\bar{\xi}_{t'})+\Omega({\xi}_{t'})$. Note that  $\Omega(\widetilde{\xi}_{t'})$ has the same dimension as $\nabla_\theta \mathcal{L}(\theta_{t'})$.\\
Taking $L_2$-norm from LHS and RHS of (\ref{eq:thm2-6}) and using $\|a\|^2=\|-a\|^2$, leads to
\begin{align}\label{eq:thm2-7}
&\|\mathbf{P}_{t\rightarrow t'}\odot\nabla_\theta \mathcal{L}(\theta_{t'})\|^2=\|\frac{1}{n_{t'}}\sum_{i=1}^{n_{t'}} \mathbf{P}^{(i)}_{t\rightarrow t'}\odot\nabla_\theta S^{id}_{t',y_i}\nonumber\\
&+\frac{\lambda}{2n_{t'}}\sum_{i=1}^{n_{t'}} \left[\mathbf{P}^{(i)}_{t\rightarrow t'}\odot\nabla_\theta S^{id}_{t',z_i\rightarrow u_i}+\mathbf{P}^{(i)}_{t\rightarrow t'}\odot\nabla_\theta S^{id}_{t',u_i\rightarrow z_i}\right]-\Omega(\widetilde{\xi}_{t'})\|^2\nonumber\\
&\geq \frac{1}{n_{t'}}\sum_{i=1}^{n_{t'}} \|\mathbf{P}^{(i)}_{t\rightarrow t'}\odot\nabla_\theta S^{id}_{t',y_i}\|^2\nonumber\\
&+\frac{\lambda}{2n_{t'}}\sum_{i=1}^{n_{t'}} \left(\|\mathbf{P}^{(i)}_{t\rightarrow t'}\odot\nabla_\theta S^{id}_{t',z_i\rightarrow u_i}\|^2+\|\mathbf{P}^{(i)}_{t\rightarrow t'}\odot\nabla_\theta S^{id}_{t',u_i\rightarrow z_i}\|^2\right)+\|\Omega(\widetilde{\xi}_{t'})\|^2 -\rho_{max},
\end{align}
where $\rho_{max}\in\mathbb{R}$ is the maximum of the remainder terms. We simplify the RHS of (\ref{eq:thm2-7}) as
\begin{align}
&\frac{1}{2n_{t'}}\sum_{i=1}^{n_{t'}} \left(\|\mathbf{P}^{(i)}_{t\rightarrow t'}\odot\nabla_\theta S^{id}_{t',y_i}\|^2+\lambda\|\mathbf{P}^{(i)}_{t\rightarrow t'}\odot\nabla_\theta S^{id}_{t',z_i\rightarrow u_i}\|^2+\lambda\|\mathbf{P}^{(i)}_{t\rightarrow t'}\odot\nabla_\theta S^{id}_{t',u_i\rightarrow z_i}\|^2\right)\nonumber\\
&+\Omega(\rho_{max},\widetilde{\xi}_{t'}), 
\end{align}
where $\Omega(\rho_{max},\widetilde{\xi}_{t'})=\|\Omega(\widetilde{\xi}_{t'})\|^2 -\rho_{max}$. This implies the alternative optimization problem in (\ref{thm:CBF-optimization}) and concludes our claim in Theorem~\ref{thm2:BF-Gradient}. 
\end{proof}

\subsection{Proof of Theorem~\ref{thm3:FPR-CBF-Optimal}}
\label{app:disscuss_thm_3}

\begin{proof}
Apply BF on the gradient of $\widehat{FPR}_{t\rightarrow t'} (\tau,\lambda)$:
\begin{align}\label{eq:thm3-1}
\mathbf{P}_{t\rightarrow t'}\odot\nabla_\theta \widehat{FPR}_{t\rightarrow t'} (\tau,\lambda)&=\frac{1}{\beta \; n^{ood}_{t'}} \sum_{i=1}^{n^{ood}_{t'}} \mathbf{P}^{(i)}_{t\rightarrow t'}\odot\left(\sigma'\left(\frac{S_{MCM}(\mathcal{I}^{ood}_{i, t\rightarrow t'})-\lambda}{\beta}\right) \nabla_\theta S_{MCM}(\mathcal{I}^{ood}_{i, t\rightarrow t'})\right)\nonumber \\
&\geq \frac{C_\sigma}{\beta \; n^{ood}_{t'}} \sum_{i=1}^{n^{ood}_{t'}} \mathbf{P}^{(i)}_{t\rightarrow t'}\odot\nabla_\theta S_{MCM}(\mathcal{I}^{ood}_{i, t\rightarrow t'}). 
\end{align}
Note that $\sigma'(u)=\sigma(u)(1-\sigma(u))$ and lower bounded by $C_\sigma$. Next, let us focus on $\nabla_\theta S_{MCM}$. Now take $\|.\|^2$:
\begin{align}
\|\mathbf{P}_{t\rightarrow t'}\odot\nabla_\theta \widehat{FPR}_{t\rightarrow t'} (\tau,\lambda)\|^2\geq \frac{C_\sigma}{\beta \; n^{ood}_{t'}} \sum_{i=1}^{n^{ood}_{t'}} \|\mathbf{P}^{(i)}_{t\rightarrow t'}\odot\nabla_\theta S_{MCM}(\mathcal{I}^{ood}_{i, t\rightarrow t'})\|^2-\rho_{max, t\rightarrow t'}, 
\end{align}
where $\rho_{max, t\rightarrow t'}$ is the maximum of the remainder terms.  

Let $k$ be the class that maximizes the $S_{MCM}$: 
\begin{align}\label{eq:thm3-6}
\nabla_\theta S_{MCM} &= \nabla_\theta \left\{\frac{e^{s_k(\mathcal{I}_{t'}^{ood})/\tau}}{\sum_{j=1}^C e^{s_j(\mathcal{I}_{t'}^{ood})/\tau}}+\mathbf{A}^*_{t\rightarrow t'}\frac{e^{s_k(\mathcal{I}_{t\rightarrow t'}^{ood})/\tau}}{\sum_{j=1}^C e^{s_j(\mathcal{I}_{t\rightarrow t'}^{ood})/\tau}}\right\}\nonumber\\
&=\nabla_\theta \left\{\frac{e^{s_k(\mathcal{I}_{t'}^{ood})/\tau}}{\sum_{j=1}^C e^{s_j(\mathcal{I}_{t'}^{ood})/\tau}}\right\}+ \mathbf{A}^*_{t\rightarrow t'}  \nabla_\theta \left\{\frac{e^{s_k(\mathcal{I}_{t\rightarrow t'}^{ood})/\tau}}{\sum_{j=1}^C e^{s_j(\mathcal{I}_{t\rightarrow t'}^{ood})/\tau}}\right\}.
\end{align}

By (\ref{Env.alignment-factor}) and Def.~\ref{def:Cross-environment alignment}, $\mathbf{A}^*_{t\rightarrow t'} = \pm(2-\sigma_{\mathcal{I}}-\sigma_{\mathcal{T}})$.
Set 
\begin{align}\label{thm3:eq2}
p_{k, t'}:=\frac{e^{s_k(\mathcal{I}_{t'}^{ood})/\tau}}{\sum_{j=1}^C e^{s_j(\mathcal{I}_{t'}^{ood})/\tau}},\;\;\; p_{k,t\rightarrow t'}:=\frac{e^{s_k(\mathcal{I}_{t\rightarrow t'}^{ood})/\tau}}{\sum_{j=1}^C e^{s_j(\mathcal{I}_{t\rightarrow t'}^{ood})/\tau}}.
\end{align}
The by chain rule, the gradient of $S_{MCM}$ is the following 
\begin{align}
\frac{\partial p_{k, t'}}{\partial z_{j, t'}}= p_{k, t'}(\delta_{kj, t'}-p_{j, t'}), \;\; \frac{\partial p_{k, t\rightarrow t'}}{\partial z_{j, t\rightarrow t'}}= p_{k, t\rightarrow t'}(\delta_{kj, t\rightarrow t'}-p_{j, t\rightarrow t'}), 
\end{align}
where $z_i=s_j(\mathcal{I}_{t'}^{ood})/\tau$. 
Here $p_{j, t'}$ and $p_{j, t\rightarrow t'}$ are defined similar to (\ref{thm3:eq2}). $\delta_{kj}$ and $\delta_{kj, t\rightarrow t'}$ are Kronecker delta. In addition
\begin{align}\label{thm3-5}
&\nabla_\theta \left\{\frac{e^{s_k(\mathcal{I}_{t'}^{ood})/\tau}}{\sum_{j=1}^C e^{s_j(\mathcal{I}_{t'}^{ood})/\tau}}\right\}=\sum_{j=1}^C \frac{\partial p_{k, t'}}{\partial z_{j, t'}}\nabla_\theta z_{j, t'}\nonumber\\
&= \sum_{j=1}^C p_{k, t'}(\delta_{kj,t}-p_{j,t'})\nabla_\theta z_{j, t'}= p_{k, t'}\left(\nabla_\theta z_{k, t'}-\sum_{j=1}^C p_{j, t'} \nabla_\theta z_{j, t'} \right)=\frac{p_{k, t'}}{\tau} \left(\nabla_\theta s_{k, t'}-\sum_{j=1}^C p_{j, t'} \nabla_\theta s_{j, t'} \right)
\end{align}
Similarly

\begin{align}\label{thm3-4}
\nabla_\theta  \left\{\frac{e^{s_k(\mathcal{I}_{t\rightarrow t'}^{ood})/\tau}}{\sum_{j=1}^C e^{s_j(\mathcal{I}_{t\rightarrow t'}^{ood})/\tau}}\right\}&= p_{k, t\rightarrow t'}\left(\nabla_\theta z_{k, t\rightarrow t'}-\sum_{j=1}^C p_{j, t\rightarrow t'} \nabla_\theta z_{j, t\rightarrow t'} \right)\nonumber\\
&=\frac{p_{k, t\rightarrow t'}}{\tau}\left(\nabla_\theta s_{k, t\rightarrow t'}-\sum_{j=1}^C p_{j, t\rightarrow t'} \nabla_\theta s_{j, t\rightarrow t'} \right).
\end{align}
Combining (\ref{thm3-4}) and (\ref{thm3-5}), bounding $\frac{p_{k, t'}}{\tau}\geq C_{k,t'}$ and $\frac{p_{k, t\rightarrow t'}}{\tau}\geq C_{k,t\rightarrow t'}$ and plugging them in (\ref{eq:thm3-6}), we have 
\begin{align}\label{thm3-7}
&\|\mathbf{P}_{t\rightarrow t'}\odot\nabla_\theta \widehat{FPR}_{t\rightarrow t'} (\tau,\lambda)\|^2\geq \frac{C_\sigma C_k}{\beta \; n^{ood}_{t'}} \sum_{i=1}^{n^{ood}_{t'}} \large\|\mathbf{P}^{(i)}_{t\rightarrow t'}\odot [\nabla_\theta s_{k, t'}-\sum_{j=1}^C p_{j, t'} \nabla_\theta s_{j, t'} ]\nonumber\\
&+\mathbf{A}^*_{t\rightarrow t'} \mathbf{P}^{(i)}_{t\rightarrow t'}\odot[\nabla_\theta s_{k, t\rightarrow t'}-\sum_{j=1}^C p_{j, t\rightarrow t'} \nabla_\theta s_{j, t\rightarrow t'}] \large\|^2-\rho_{max, t\rightarrow t'}, 
\end{align}
$C_k=min\{C_{k,t'}, C_{k,t\rightarrow t'}\}$. Next we can bound the RHS of (\ref{thm3-7}) further by
\begin{align}
&\frac{C_{\sigma,k}}{\beta \; n^{ood}_{t'}} \sum_{i=1}^{n^{ood}_{t'}} \large\|\mathbf{P}^{(i)}_{t\rightarrow t'}\odot [\nabla_\theta s_{k, t'}-\sum_{j=1}^C p_{j, t'} \nabla_\theta s_{j, t'} ]\large\|^2\nonumber\\
&+\mathbf{A}^*_{t\rightarrow t'} \large\|\mathbf{P}^{(i)}_{t\rightarrow t'}\odot[\nabla_\theta s_{k, t\rightarrow t'}-\sum_{j=1}^C p_{j, t\rightarrow t'} \nabla_\theta s_{j, t\rightarrow t'}] \large\|^2-\hat{\rho}_{max, t\rightarrow t'}, 
\end{align}
where $\hat{\rho}_{max, t\rightarrow t'}$ is maximum remainder and $C_{\sigma,k}=C_\sigma C_k$. 
\end{proof}

\subsection{Proof of Theorem~\ref{thm3:MCM}}
For two environments $\mathcal{E}_t$, $\mathcal{E}_{t'}$ ($t<t'$) with label sets $\mathcal{Y}_{in,t} = \{y_{t,1}, y_{t,2}, ..., y_{t,C} \}$ and $\mathcal{Y}_{in,t'} = \{y_{t',1}, y_{t',2}, ..., y_{t',C} \}$ and collection of concepts vectors $\mathcal{T}_{t,i}$, $\mathcal{T}_{t',i}$, for $i=1,\ldots,C$. Without loss of generality we assume that both environments have the same class size. We define the cross-environment maximum concept matching (MCM) score \begin{align}\label{app:def.cross-environment-MCM-score}
S_{MCM}(\mathcal{I}^{test}_{t\rightarrow t'})= \max_i\left\{\frac{e^{s_i(\mathcal{I}_{t'}^{test})/\tau}}{\sum_{j=1}^C e^{s_j(\mathcal{I}_{t'}^{test})/\tau}}+\mathbf{A}^*_{t\rightarrow t'}\frac{e^{s_i(\mathcal{I}_{t\rightarrow t'}^{test})/\tau}}{\sum_{j=1}^C e^{s_j(\mathcal{I}_{t\rightarrow t'}^{test})/\tau}}\right\},
\end{align}
where $s_i(\mathcal{I}_{t'}^{test})$ is the cosine similarity between test image encoder $z_{t'}=f(\mathcal{I}_{t'}^{test})$ and trained text encoder in $\mathcal{E}_{t'}$ $u_{t',i}:=g(\mathcal{T}_{t'}^{train}(t_i))$.  $s_i(\mathcal{I}_{t\rightarrow t'}^{test})$ is the cosine similarity between test image encoder  $z_{t'}=f(\mathcal{I}_{t'}^{test})$ in $\mathcal{E}_{t'}$ and train text encoder in $\mathcal{E}_{t}$ $u_{t,i}:=g(\mathcal{T}_{t}^{train}(t_i))$. The environment alignment factor $\mathbf{A}^*_{t\rightarrow t'}$ is determined by intra-modal contribution in Def.~\ref{def:Intra-Modality Contribution}: 
\begin{align}\label{Env.alignment-factor-SM}
\mathbf{A}^*_{t\rightarrow t'}= \left\{ \begin{array}{rcl}
+Con_{\mathcal{E}_t\rightarrow \mathcal{E}_{t'}} & \hbox{if $\mathcal{E}_t$ and $\mathcal{E}_{t'}$ are aligned}\;\; \\
\\ -Con_{\mathcal{E}_t\rightarrow \mathcal{E}_{t'}}  & \hbox{if $\mathcal{E}_t$ and $\mathcal{E}_{t'}$ are not aligned},  
\end{array}\right.
\end{align}
where $Con_{\mathcal{E}_t\rightarrow \mathcal{E}_{t'}}=Con_{\mathcal{I}_t\rightarrow\mathcal{I}_{t'}} +Con_{\mathcal{T}_t\rightarrow\mathcal{T}_{t'}}$ and $\mathcal{E}_t$ and $\mathcal{E}_{t'}$ are aligned iff $Con_{\mathcal{E}_t\rightarrow \mathcal{E}_{t'}}> \alpha$. 
The MCM-based OOD detector is 
\label{app:thm3:MCM}
\begin{align}
D(\mathcal{I}_{t'}^{test})= \left\{ \begin{array}{rcl}
1 & \hbox{if}\;\; S_{MCM}(\mathcal{I}^{test}_{t\rightarrow t'}) \geq \lambda\\
\\ 0  & \hbox{if} \;S_{MCM}(\mathcal{I}^{test}_{t\rightarrow t'}) < \lambda
\end{array}\right.
\end{align}
\begin{proof}
Let $z:=\mathds{1}\{y_{t'}\in\mathcal{Y}_{in,{t'}}\}$ and $\mathbb{P}^{ood}_{\mathcal{I}^{test}_{t'}}:\mathbb{P}_{\mathcal{I}^{test}_{t'}|z=0}$(i.e the marginal
distribution of $\mathcal{I}_{t'}$ conditioned on $z= 0$) denotes the out-of-distribution for image $\mathcal{I}_{t'}$ .  By definition, the false positive rate $FPR(\tau,\lambda)$ is computed as 
\begin{align}
FPR(\tau,\lambda)=\mathbb{P}(D(\mathcal{I}_{t'}^{test})=1|z=0)=\mathbb{P}^{ood}_{\mathcal{I}^{test}_{t'}} (D(\mathcal{I}_{t'}^{test})=1)=\mathbb{P}^{ood}_{\mathcal{I}^{test}_{t'}} (P_{\hat{y}_{t'}}(\mathcal{I}_{t'};\tau)\geq \lambda)
\end{align}
where $\hat{y}_{t\rightarrow t'}:= \argmax_{i\in [C]}s_i(\mathcal{I}^{test}_{t\rightarrow t'})$ 
is largest cross-environment cosine similarities for an OOD input $\mathcal{I}$. And $\hat{y}_{t'}:= \argmax_{i\in [C]}s_i(\mathcal{I}^{test}_{t'})$ is the largest intra-environment cosine similarities for an OOD input $\mathcal{I}$. Note that $\hat{y}_{t'}$ and $\hat{y}_{t\rightarrow t'}$ are different classes. 

Going back to (\ref{app:def.cross-environment-MCM-score}), we write 
\begin{align}
FPR_{t\rightarrow t'}(\tau,\lambda)&=\mathbb{P}^{ood}_{\mathcal{I}^{test}_{t'}} \left(\frac{e^{s_{\hat{y}_{t\rightarrow t'}}(\mathcal{I}_{t'}^{test})/\tau}}{\sum_{j=1}^C e^{s_j(\mathcal{I}_{t'}^{test})/\tau}}+\mathbf{A}^*_{t\rightarrow t'}\frac{e^{s_{\hat{y}_{t\rightarrow t'}}(\mathcal{I}_{t\rightarrow t'}^{test})/\tau}}{\sum_{j=1}^C e^{s_j(\mathcal{I}_{t\rightarrow t'}^{test})/\tau}}\geq \lambda\right)\nonumber\\
&\leq \mathbb{P}^{ood}_{\mathcal{I}^{test}_{t'}} \left(\frac{e^{s_{\hat{y}_{t\rightarrow t'}}(\mathcal{I}_{t'}^{test})/\tau}}{\sum_{j=1}^C e^{s_j(\mathcal{I}_{t'}^{test})/\tau}}+\mathbf{A}^*_{t\rightarrow t'}\frac{e^{s_{\hat{y}_{t\rightarrow t'}}(\mathcal{I}_{t\rightarrow t'}^{test})/\tau}}{\sum_{j=1}^C e^{s_j(\mathcal{I}_{t'}^{test})/\tau}}\geq \lambda\right)\label{eq:222}\\
&=\mathbb{P}^{ood}_{\mathcal{I}^{test}_{t'}} \left(\frac{1}{\lambda}\geq \sum_{i=1}^C \exp\left(\frac{s_j(\mathcal{I}_{t'}^{test})-s_{\hat{y}_{t\rightarrow t'}}(\mathcal{I}_{t'}^{test})-\mathbf{A}^*_{t\rightarrow t'}s_{\hat{y}_{t\rightarrow t'}}(\mathcal{I}_{t\rightarrow t'}^{test})}{\tau}\right)\right). 
\end{align}
The inequality in (\ref{eq:222}) is based on Assumption~\ref{ass:positive environment transition}. Note that because of $\mathcal{E}_t$ and $\mathcal{E}_{t'}$ are aligned therefore $\mathbf{A}^*_{t\rightarrow t'}$ is positive. By inequality $e^u \geq 1 + u$, we have,
\begin{align}\label{Ineq.2}
\mathbb{P}^{ood}_{\mathcal{I}^{test}_{t'}} \left(\frac{1}{\lambda}\geq \sum_{i=1}^C \exp\left(\frac{s_j(\mathcal{I}_{t'}^{test})-s_{\hat{y}_{t\rightarrow t'}}(\mathcal{I}_{t'}^{test})-\mathbf{A}^*_{t\rightarrow t'}s_{\hat{y}_{t\rightarrow t'}}(\mathcal{I}_{t\rightarrow t'}^{test})}{\tau}\right)\right) \nonumber\\
\leq \mathbb{P}^{ood}_{\mathcal{I}^{test}_{t'}} \left(\frac{1}{\lambda}\geq \sum_{i=1}^C \left(1+\frac{s_j(\mathcal{I}_{t'}^{test})-s_{\hat{y}_{t\rightarrow t'}}(\mathcal{I}_{t'}^{test})-\mathbf{A}^*_{t\rightarrow t'}s_{\hat{y}_{t\rightarrow t'}}(\mathcal{I}_{t\rightarrow t'}^{test})}{\tau}\right)\right)\nonumber\\
=\mathbb{P}^{ood}_{\mathcal{I}^{test}_{t'}} \left(\sum_{i=1}^C\left(s_{\hat{y}_{t\rightarrow t'}}(\mathcal{I}_{t'}^{test})+\mathbf{A}^*_{t\rightarrow t'}s_{\hat{y}_{t\rightarrow t'}}(\mathcal{I}_{t\rightarrow t'}^{test})-s_j(\mathcal{I}_{t'}^{test})\right)\geq \left(C-\frac{1}{\lambda}\right)\tau\right)\nonumber\\
=\mathbb{P}^{ood}_{\mathcal{I}^{test}_{t'}} \left(\sum_{i=1}^C\left(s_{\hat{y}_{t\rightarrow t'}}(\mathcal{I}_{t'}^{test})-s_j(\mathcal{I}_{ t'}^{test})\right)+\sum_{i=1}^C\mathbf{A}^*_{t\rightarrow t'}s_{\hat{y}_{t\rightarrow t'}}(\mathcal{I}_{t\rightarrow t'}^{test})\geq \left(C-\frac{1}{\lambda}\right)\tau\right)
\end{align}

Next we have
\begin{align}
&\sum_{i=1}^C\left(s_{\hat{y}_{t\rightarrow t'}}(\mathcal{I}_{t'}^{test})-s_j(\mathcal{I}_{t'}^{test})\right)=\sum_{i=1}^C\left(s_{\hat{y}_{t\rightarrow t'}}(\mathcal{I}_{t'}^{test})- s_{\hat{y}_{t'}}(\mathcal{I}_{t'}^{test})+ s_{\hat{y}_{t'}}(\mathcal{I}_{t'}^{test})-s_j(\mathcal{I}_{t'}^{test})\right)\nonumber\\
&=\sum_{i=1}^C\left(s_{\hat{y}_{t\rightarrow t'}}(\mathcal{I}_{t'}^{test})- s_{\hat{y}_{t'}}(\mathcal{I}_{t'}^{test})\right)+ \sum_{i=1}^C\left( s_{\hat{y}_{t'}}(\mathcal{I}_{t'}^{test})-s_j(\mathcal{I}_{t'}^{test})\right)\nonumber\\
&=C \left(s_{\hat{y}_{t\rightarrow t'}}(\mathcal{I}_{t'}^{test})-s_{\hat{y}_{t'}}(\mathcal{I}_{t'}^{test})\right)+\sum_{i=1}^ C\left(s_{\hat{y}_{t'}}(\mathcal{I}_{t'}^{test})-s_j(\mathcal{I}_{ t'}^{test})\right)
\end{align}
Note that $s_{\hat{y}_{t\rightarrow t'}}(\mathcal{I}_{t'}^{test})-s_{\hat{y}_{t'}}(\mathcal{I}_{t'}^{test})<0$.  Using Assumption~\ref{ass:Optimal Class Connections} (a), 
\begin{align}
\mathbb{P}^{ood}_{\mathcal{I}^{test}_{t'}} \left(\sum_{i=1}^C\left(s_{\hat{y}_{t\rightarrow t'}}(\mathcal{I}_{t'}^{test})-s_j(\mathcal{I}_{t'}^{test})\right)<  C \left(s_{\hat{y}_{t\rightarrow t'}}(\mathcal{I}_{t'}^{test})-s_{\hat{y}_{t'}}(\mathcal{I}_{t'}^{test})\right)+ C\delta\right)=1
\end{align}
This implies that 
\begin{align}
&\mathbb{P}^{ood}_{\mathcal{I}^{test}_{t'}} \left(\sum_{i=1}^C\left(s_{\hat{y}_{t\rightarrow t'}}(\mathcal{I}_{t'}^{test})-s_j(\mathcal{I}_{ t'}^{test})\right)+\sum_{i=1}^C\mathbf{A}^*_{t\rightarrow t'}s_{\hat{y}_{t\rightarrow t'}}(\mathcal{I}_{t\rightarrow t'}^{test})\geq \left(C-\frac{1}{\lambda}\right)\tau\right)\nonumber\\
&\leq \mathbb{P}^{ood}_{\mathcal{I}^{test}_{t'}} \left(C\mathbf{A}^*_{t\rightarrow t'}s_{\hat{y}_{t\rightarrow t'}}(\mathcal{I}_{t\rightarrow t'}^{test})+C \left(s_{\hat{y}_{t\rightarrow t'}}(\mathcal{I}_{t'}^{test})-s_{\hat{y}_{t'}}(\mathcal{I}_{t'}^{test})\right)\geq \left(C-\frac{1}{\lambda}\right)\tau-C\delta\right)\nonumber\\
&=\mathbb{P}^{ood}_{\mathcal{I}^{test}_{t'}} \left(\mathbf{A}^*_{t\rightarrow t'}s_{\hat{y}_{t\rightarrow t'}}(\mathcal{I}_{t\rightarrow t'}^{test})+ \left(s_{\hat{y}_{t\rightarrow t'}}(\mathcal{I}_{t'}^{test})-s_{\hat{y}_{t'}}(\mathcal{I}_{t'}^{test})\right)\geq \lambda'-\delta\right), \label{Thm:eq2}
\end{align}
where $\lambda'=\left(1-\frac{1}{C\lambda}\right)\tau$. Recalling Definition~\ref{def:Intra-Modality Contribution} and the environment alignment factor (\ref{Env.alignment-factor-SM}), we have $\mathbf{A}^*_{t\rightarrow t'}=2-\sigma_\mathcal{I}-\sigma_{\mathcal{T}}$. Under Assumption~\ref{ass:Optimal Class Connections} (b), $\exists M\geq 0$ such that

\begin{align}
\hbox{RHS of (\ref{Thm:eq2})} &\leq \mathbb{P}^{ood}_{\mathcal{I}^{test}_{t'}} \left( Ms_{\hat{y}_{t'}}(\mathcal{I}_{t'}^{test})\geq \lambda'-\delta\right)\\
&= \mathbb{P}^{ood}_{\mathcal{I}^{test}_{t'}} \left( s_{\hat{y}_{t'}}(\mathcal{I}_{t'}^{test})/\tau\geq \frac{1}{M\tau}(\lambda'-\delta)\right)\nonumber\\
&\leq \mathbb{P}^{ood}_{\mathcal{I}^{test}_{t'}} \left( \frac{e^{s_{\hat{y}_{t'}}(\mathcal{I}_{t'}^{test})/\tau}}{\sum_{j=1}^C e^{s_j(\mathcal{I}_{t'}^{test})/\tau}}\geq \frac{\tau}{M'}(\lambda'-\delta)\right).\label{them:eq34}
\end{align}
where $\hat{y}_{t'}:= \argmax_{i\in [C]}s_i(\mathcal{I}^{test}_{t'})$. The Ineq.~(\ref{them:eq34}) holds true because the normalized term $\sum_{j=1}^C e^{s_j(\mathcal{I}_{t'}^{test})/\tau}$ is bounded. This concludes that
$$
FPR_{t\rightarrow t'}(\tau,\lambda)\leq FPR_{t'}(\tau,\lambda_{t'}), \;\; \lambda_{t'}:= \frac{\tau}{M'}(\lambda'-\delta).
$$
\end{proof}

\subsection{Proofs for Section~\ref{sec:theory}: hard and soft OOD}\label{app:hard-soft}
Throughout we use Cantelli's one-sided inequality \citep{cantelli1928confini}: for a random variable $W$ with mean $\mu$ and variance $v$, and any $a>0$, $\Pr[W-\mu\le -a]\le v/(v+a^{2})$ and $\Pr[W-\mu\ge a]\le v/(v+a^{2})$. We also use the following bound for variables bounded above: if $W\le 1$ has mean $\mu$ and $\ell<1$, then $\mu\le\ell\Pr[W<\ell]+\Pr[W\ge\ell]=\ell+(1-\ell)\Pr[W\ge\ell]$, so
\begin{equation}\label{eq:reverse-markov}
\Pr[W\ge\ell]\ge\frac{\mu-\ell}{1-\ell}.
\end{equation}

\begin{proof}[Proof of Lemma~\ref{lem:ood-moments}]
By (\ref{eq:mcm}), $\tau S_t(x)=\max_{c}\cos(u_{t,c},f(x))$; taking expectations under $P^I_t$ and $P^I_s$ gives $m_t$ and $\mathrm{Con}^{\max}_{\mathcal{T}_t\to\mathcal{I}_s}$ by Definition~\ref{def:gencon}. A maximum over the classes dominates their uniform average pointwise, which gives $\mathrm{Con}^{\max}_{\mathcal{T}_t\to\mathcal{I}_s}\ge\mathrm{Con}_{\mathcal{T}_t\to\mathcal{I}_s}$.
\end{proof}

\begin{proof}[Proof of Theorem~\ref{thm:soft-ood}]
Let $D=\tau\big(S_t(X)-S_t(X')\big)$. By independence, $\mathbb{E}[D]=\gamma_{t,s}\ge\eta>0$ and $\mathrm{Var}(D)=\nu_{t,s}$, so Cantelli gives $\Pr[D\le 0]\le\nu_{t,s}/(\nu_{t,s}+\gamma_{t,s}^{2})$ and $\mathrm{AUROC}_t(s)\ge\Pr[D>0]\ge\gamma_{t,s}^{2}/(\gamma_{t,s}^{2}+\nu_{t,s})$. Since $g\mapsto g^{2}/(g^{2}+\nu)$ is non-decreasing on $g>0$, this is at least $\eta^{2}/(\eta^{2}+\nu_{t,s})$, and $\nu_{t,s}\le 2$ gives the last inequality. For the rates, $\tau S_t(X)$ has mean $m_t$ and variance $v^{id}_t$, so $1-\mathrm{TPR}_t(\lambda)=\Pr[\tau S_t(X)<m_t-\eta/2]\le v^{id}_t/(v^{id}_t+\eta^{2}/4)$. The variable $\tau S_t(X')$ has mean $\mu'\le m_t-\eta$, so $\mathrm{FPR}_t(s;\lambda)\le\Pr[\tau S_t(X')-\mu'\ge\eta/2]\le v^{ood}_{t,s}/(v^{ood}_{t,s}+\eta^{2}/4)$. The outer bounds use $v^{id}_t,v^{ood}_{t,s}\le 1$.
\end{proof}

\begin{proof}[Proof of Theorem~\ref{thm:hard-ood}]
The first claim follows from Definition~\ref{def:hard-soft} and the monotonicity of $g\mapsto g^{2}/(g^{2}+\nu)$. For (\ref{eq:hard-rates}), apply (\ref{eq:reverse-markov}) to $W=\tau S_t(X')$, whose mean is $\mathrm{Con}^{\max}_{\mathcal{T}_t\to\mathcal{I}_s}\ge m_t-\eta$, and to $W=\tau S_t(X)$, whose mean is $m_t$ (Lemma~\ref{lem:ood-moments}); note that $S_t\ge\lambda$ if and only if $\tau S_t\ge\ell$. If $\mathcal{E}_s$ is strictly hard, then with $D$ as above $\mathbb{E}[D]=\gamma_{t,s}<0$, so $\mathrm{AUROC}_t(s)\le\Pr[D\ge 0]=\Pr[D-\gamma_{t,s}\ge|\gamma_{t,s}|]\le\nu_{t,s}/(\nu_{t,s}+\gamma_{t,s}^{2})$ by Cantelli, and $|\gamma_{t,s}|=\mathrm{Con}^{\max}_{\mathcal{T}_t\to\mathcal{I}_s}-m_t$ by (\ref{eq:gap-contribution}).
\end{proof}

\begin{proof}[Proof of Proposition~\ref{prop:align-hard}]
Let $a,b,c$ be the unit vectors along $u_{t,y}$, $f(x)$, and $f(x')$. For unit vectors $\|a-b\|^{2}=2(1-\cos(a,b))$, so the triangle inequality $\|a-c\|\le\|a-b\|+\|b-c\|$ gives, pointwise,
\begin{equation*}
1-\cos(a,c)\le\Big(\sqrt{1-\cos(a,b)}+\sqrt{1-\cos(b,c)}\Big)^{2}.
\end{equation*}
Expanding the square, taking expectations, and applying the Cauchy--Schwarz inequality $\mathbb{E}[\sqrt{UV}]\le\sqrt{\mathbb{E}[U]\,\mathbb{E}[V]}$ to the cross term yields $\mathbb{E}[1-\cos(a,c)]\le\big(\sqrt{\delta_t}+\sqrt{\sigma_{\mathcal{I}}}\big)^{2}$. Finally, $\max_{c'}\cos(u_{t,c'},f(x'))\ge\cos(u_{t,y},f(x'))$ pointwise and $x'$ has marginal law $P^I_s$, so $\mathrm{Con}^{\max}_{\mathcal{T}_t\to\mathcal{I}_s}\ge\mathbb{E}[\cos(a,c)]$.
\end{proof}

\begin{proof}[Proof of Corollary~\ref{cor:aligned-hard}]
Here $\sigma_{\mathcal{I}}\le\bar\sigma$, and the right-hand side of (\ref{eq:align-hard}) is non-increasing in $\sigma_{\mathcal{I}}$; combine with Definition~\ref{def:hard-soft}.
\end{proof}

\subsection{Proofs for Section~\ref{sec:theory}: cascaded detection}\label{app:cascade}
\begin{proof}[Proof of Proposition~\ref{prop:pollution}]
Since $U_k\subseteq U_{k+1}$ and a maximum over a larger set is no smaller, $S^{(k)}(x)\le S^{(k+1)}(x)$; hence $\{S^{(k)}(X')\ge\lambda\}\subseteq\{S^{(k+1)}(X')\ge\lambda\}$, which gives the monotonicity of the false-positive rate. For each $j\le k$, $U_k$ contains the prototypes of $\mathcal{E}_{\pi_j}$, so $\tau S^{(k)}(x)\ge\max_c\cos(u_{\pi_j,c},f(x))$ pointwise; taking $\mathbb{E}_{x\sim P^I_{\pi_s}}$ gives $\mathrm{Con}^{\max}_{\mathcal{T}_{\pi_j}\to\mathcal{I}_{\pi_s}}$, and $j=s$ gives $m_{\pi_s}$.
\end{proof}

\begin{proof}[Proof of Theorem~\ref{thm:env-failure}]
By Proposition~\ref{prop:pollution}, $\tau\mathbb{E}[S^{(k)}(X')]\ge\rho_{k,s}$, so $\gamma^{(k)}_s=M_k-\tau\mathbb{E}[S^{(k)}(X')]\le M_k-\rho_{k,s}$; moreover $M_k\le 1$ since $\tau S^{(k)}\le 1$, and $M_k\ge m_{\pi_k}$ since $U_k$ contains the prototypes of $\mathcal{E}_{\pi_k}$. If $\rho_{k,s}>M_k$, then $\gamma^{(k)}_s<0$ with $|\gamma^{(k)}_s|\ge\rho_{k,s}-M_k$, and the Cantelli argument in the proof of Theorem~\ref{thm:hard-ood}, applied to $D=\tau\big(S^{(k)}(X)-S^{(k)}(X')\big)$, gives the AUROC bound. Finally, (\ref{eq:reverse-markov}) applied to $W=\tau S^{(k)}(X')$, whose mean is at least $\rho_{k,s}$, gives (\ref{eq:cascade-fpr}) at step $k$; Proposition~\ref{prop:pollution} extends it to every $k'\ge k$, and $\rho_{k,s}\ge m_{\pi_s}$.
\end{proof}

\begin{proof}[Proof of Corollary~\ref{cor:which-fail}]
The first claim is the second inequality in (\ref{eq:cascade-fpr}). The second follows from $\rho_{k,s}\ge\mathrm{Con}^{\max}_{\mathcal{T}_{\pi_j}\to\mathcal{I}_{\pi_s}}$ for each $j\neq s$ and (\ref{eq:align-hard}) applied to the pair $(\pi_j,\pi_s)$.
\end{proof}

\end{document}